\documentclass{article}

\PassOptionsToPackage{numbers, compress}{natbib}
\usepackage[main, final]{neurips_2026}

\usepackage{graphicx}
\usepackage{microtype}
\usepackage{graphicx}
\usepackage{subcaption}
\usepackage{booktabs}

\usepackage{algorithm}
\usepackage{algorithmic}

\usepackage{amsmath}
\usepackage{amssymb}
\usepackage{mathtools}
\usepackage{amsthm}
\usepackage{pifont}
\usepackage{nccmath}

\theoremstyle{plain}
\newtheorem{theorem}{Theorem}[section]
\newtheorem{proposition}[theorem]{Proposition}

\theoremstyle{definition}

\theoremstyle{remark}
\newtheorem{remark}[theorem]{Remark}

\usepackage{multirow} 

\usepackage[utf8]{inputenc} 
\usepackage[T1]{fontenc}    
\usepackage{hyperref}       
\usepackage{url}            
\usepackage{booktabs}       
\usepackage{amsfonts}       
\usepackage{nicefrac}       
\usepackage{microtype}      
\usepackage{xcolor}         

\title{Explainable Novel Category Discovery in Semantic Concept Space}

\author{
  Ifrat Ikhtear Uddin$^{1}$,
  Yang Zhou$^{2}$,
   Chaowei Zhang$^{3}$,
   Longwei Wang$^{1\dagger}$,
  \\
  \texttt{ifratikhtear.uddin@coyotes.usd.edu,yangzhou@auburn.edu,} \\
  \texttt{ cwzhang@yzu.edu.cn, longwei.wang@usd.edu} \\
  \\
  $^{1}$Department of Computer Science, University of South Dakota, USA \\
  $^{2}$Department of Computer Science and Software Engineering, Auburn University, USA
  \\
  $^{3}$School of Information and Artificial Intelligence, Yangzhou University, China
}

\begin{document}

\renewcommand{\thefootnote}{\fnsymbol{footnote}}
\footnotetext[0]{$\dagger$ Corresponding Authors.}
\renewcommand{\thefootnote}{\arabic{footnote}}

\maketitle

\begin{abstract}
Novel category discovery aims to identify unseen classes from unlabeled data by transferring knowledge from labeled categories, but most existing methods perform discovery in opaque latent feature spaces. As a result, they may separate novel categories accurately while providing little insight into what semantic evidence defines each discovered group. 
We propose \textbf{xNCD}, an explainable novel category discovery framework that performs both representation-based discovery and pseudo-label assignment directly in a structured semantic concept space. Instead of clustering arbitrary deep features, xNCD learns a label-free concept representation by aligning visual features with vision-language similarity priors from pretrained multimodal models, and then applies a unified labeled-and-unlabeled self-labeling objective over concept-space logits. This design makes each discovered category explainable by construction through stable concept signatures and instance-level concept evidence. 
Theoretically, we show that routing discovery through a semantic concept bottleneck induces a strict restriction of the feature-space hypothesis class, excluding a large family of unconstrained decision rules and biasing induced partitions toward semantically interpretable concept coordinates. 
Experiments on CIFAR-10, CIFAR-100, and CUB-200 demonstrate that xNCD preserves strong discovery performance while providing intrinsic explanations. Under task-agnostic evaluation, xNCD achieves 92.63\% overall accuracy on CIFAR-10, close to UNO’s 93.4\%, and improves CIFAR-100 overall accuracy from 73.2\% to 76.45\%, while being the only compared method that provides human-readable cluster- and instance-level explanations. 
\end{abstract}



\section{Introduction}

Novel Category Discovery addresses the problem of identifying previously unseen semantic categories from unlabeled data while leveraging supervision from a set of known categories \cite{joseph2022novel,han2021autonovel}. This setting naturally arises in realistic deployments, where labeled data are limited to a subset of categories and new categories emerge over time. Recent advances in representation learning and clustering \cite{wang2019representation,wang2014congestion,wang2021explaining,wang2011exploration,shi2019deep,wang2021improving,xiao2022looking,wang2019layer,wang2024dense,nayyem2024bridging,wang2025explainability,ranabhat2025multi,uddin2025expert,rasmussen2025ecologically,santosh2025toward,wang2026expert,wang2025explainability,zhang2026acting,wang2025bridging,rasmussen2026channel,wall2026winsor,ranabhat2025promoting,khadka2025coswin,wang2025explainability2} have led to strong empirical performance on NCD benchmarks \cite{fini2021unified,han2019learning, wang2024semantic}, demonstrating that high-quality embeddings can effectively separate known and novel categories.

Despite these successes, existing NCD methods largely operate in latent feature spaces whose dimensions lack explicit semantic meaning \cite{samek2021explaining}. As a result, while novel categories may be correctly separated, the reason why a group of samples forms a distinct category remains opaque. This limitation makes it difficult to validate discovered categories, diagnose failure modes, or integrate human knowledge into the discovery process. Such opacity is particularly problematic in scientific, safety-critical, and human-in-the-loop applications \cite{rudin2019stop,lipton2018mythos,doshi2017towards}, where interpretability and accountability are essential.

In this work, we argue that semantic interpretability should be an intrinsic property of the novel category discovery process, rather than a post hoc analysis applied after clustering. To this end, we propose an interpretable framework for NCD that performs representation learning and category discovery directly in a structured semantic concept space \cite{koh2020concept,yuksekgonul2022post}. Instead of relying on abstract embeddings, the framework learns a set of concept detectors whose activations form an explicit, human-readable intermediate representation for both labeled and unlabeled data.

A key challenge in adopting concept-based representations is the lack of concept  annotations at scale \cite{kim2018interpretability}. To address this, we used a label-free concept projection learning strategy \cite{oikarinen2023label} that aligns learned concept activations with vision–language similarity priors derived from pretrained multimodal models \cite{radford2021learning}. This enables the framework to acquire semantically 
meaningful concept detectors without requiring manual concept supervision. Novel category 
discovery is then performed through a unified cross-entropy objective operating directly 
in concept space.


Experiments on CIFAR-10, CIFAR-100, and CUB-200 demonstrate competitive 
discovery performance while providing intrinsic, human-readable explanations 
for discovered categories. Our key contributions are:

\begin{itemize}
    \item We introduce a novel category discovery framework that performs representation learning and category discovery directly in a structured semantic concept space. By mediating discovery through concept representations rather than latent embeddings, the framework produces novel categories that are intrinsically interpretable, with each category characterized by an explicit and human-readable concept signature.

    \item We formulate novel category discovery as a unified optimization problem over concept representations. Extensive experiments on CIFAR-10, CIFAR-100, and CUB-200 demonstrate competitive discovery performance while yielding stable and semantically meaningful explanations.
    \item 
     We theoretically show that introducing a semantic concept bottleneck induces a strict restriction on the hypothesis space of novel category discovery models, eliminating semantically entangled feature-space decision rules and constraining the induced partitions of unlabeled data to be expressible through human-interpretable concept coordinates.
\end{itemize}

\section{Related Works}
\label{sec:ncd}
Novel Category Discovery (NCD) addresses the problem of identifying unseen 
categories in unlabeled data while leveraging knowledge from labeled base 
categories. Early approaches KCL~\citep{hsu2017learning} and MCL~\citep{hsu2019multi} employed pairwise similarity learning but struggled with scalability. DTC~\citep{han2019learning} introduced deep transfer clustering, while RS+~\citep{han2020automatically} combined ranking statistics with pseudo-labeling. UNO~\citep{fini2021unified} introduced a unified cross-entropy objective treating labeled and pseudo-labels homogeneously, with follow-up work extending NCD to generalized~\citep{vaze2022generalized}, open-world~\citep{cao2021open}, and continual learning~\citep{joseph2022novel} settings. Despite strong clustering accuracy, all existing NCD methods operate in opaque feature spaces where discovered clusters carry no semantic meaning, preventing domain experts from validating discovered groupings.

While GCD~\citep{vaze2022generalized} represents a more general formulation, our 
task-agnostic evaluation is directly comparable at test time, as both require joint 
classification without oracle task identity. The key distinction is in training: GCD 
mixes known samples into the unlabeled pool, while xNCD operates in the more controlled 
NCD setting where interpretability constraints are easier to enforce. Notably, GCD and 
subsequent methods~\citep{wen2023parametric} rely on ViT backbones pretrained with 
DINO on ImageNet, whereas xNCD uses ResNet-18/50 pretrained only on the labeled split.

Concept Bottleneck Models (CBMs)~\citep{koh2020concept} introduce interpretability 
by routing predictions through human-understandable concept neurons, enabling 
transparent reasoning. Label-free CBM~\citep{oikarinen2023label} removes the need 
for manual annotations by aligning concepts with vision-language priors. However, 
CBMs have been studied exclusively in closed-world settings with known categories. 
\textbf{This work is the first to extend concept bottleneck architectures to 
open-world learning}, performing discovery in interpretable concept space rather 
than opaque features, and extending concept reasoning to scenarios where target 
categories are unknown.

\begin{figure*}[ht]
    \centering
    \includegraphics[width=0.85\linewidth]{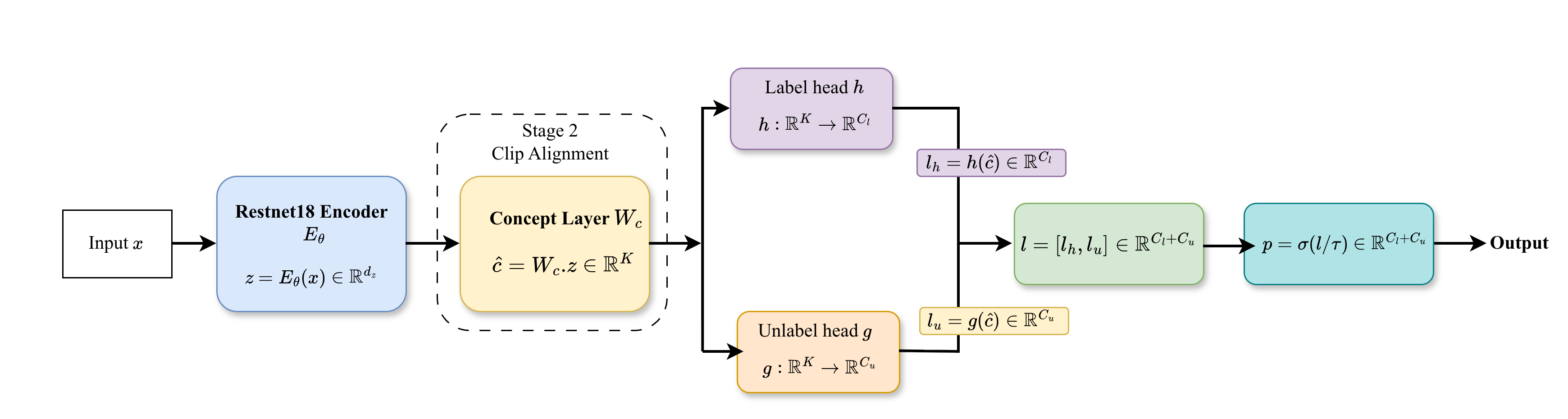}
        \caption{An input image $\mathbf{x}$ is processed by a pretrained encoder $E_\theta$ to produce features $\mathbf{z} \in \mathbb{R}^{d_z}$. The concept projection layer $\mathbf{W}_c$ maps these features to interpretable concept activations $\hat{\mathbf{c}} \in \mathbb{R}^K$, where each dimension corresponds to a human-understandable attribute. Two classification heads operate on concept activations: a labeled head $h$ for known categories and an unlabeled head $g$ for novel categories discovery. Logits are concatenated and passed through a unified softmax over all $C_l + C_u$ categories.} 

    \label{fig:architecture}
\end{figure*}

\begin{figure*}[ht]
    \centering
    \includegraphics[width=0.8\linewidth]{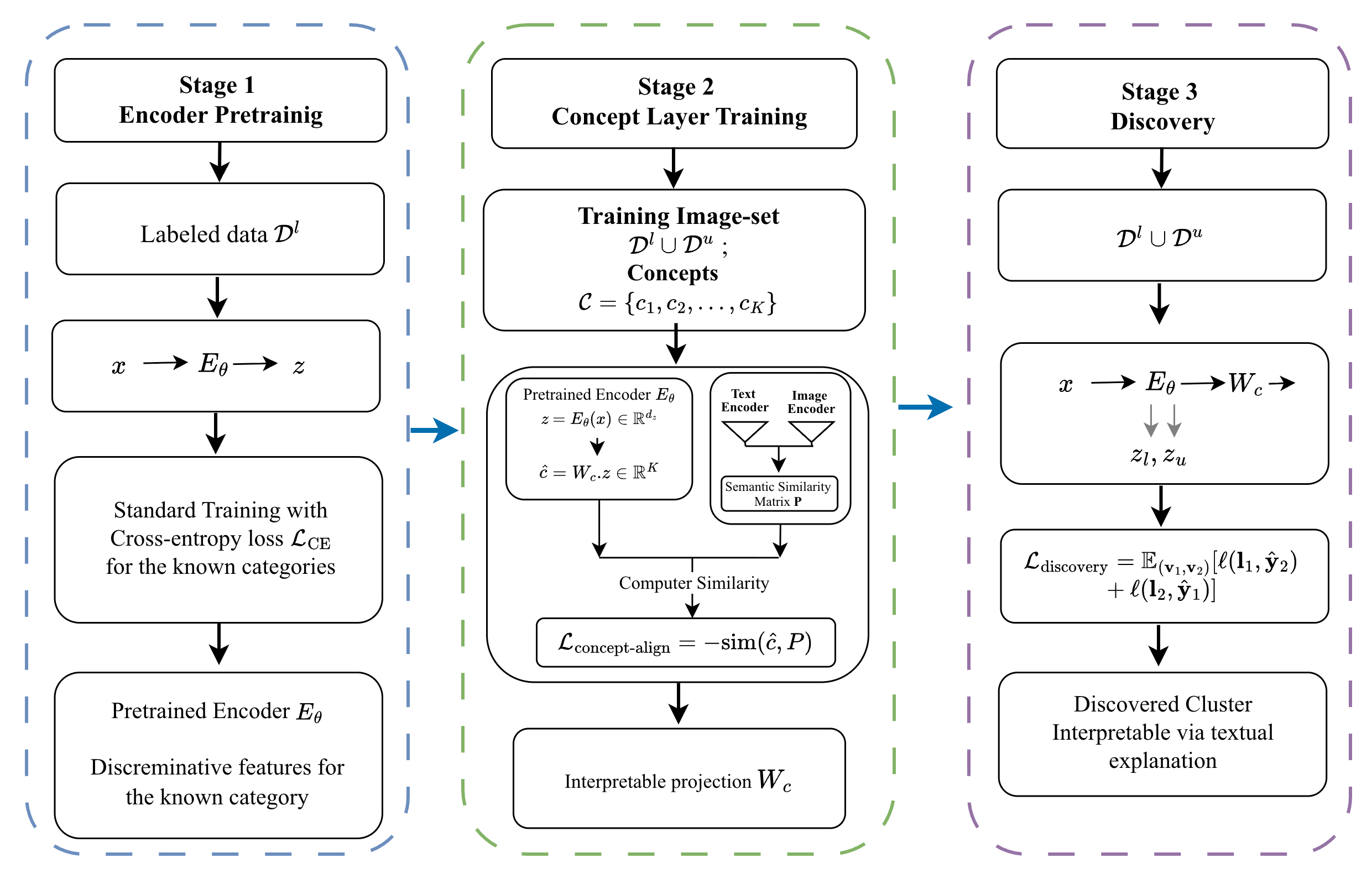}
\caption{Overview of the xNCD framework. \textbf{Stage 1}: Encoder pretraining 
on labeled data with cross-entropy loss. \textbf{Stage 2}: Concept projection 
$\mathbf{W}_c$ is trained on all images by aligning concept activations with 
CLIP image-text similarities no concept annotations required. \textbf{Stage 3}: 
Novel categories are discovered by clustering in concept space, yielding 
interpretable clusters characterized by human-readable concept profiles 
(e.g., ``four-legged, furry, wide mouth'' for a dog cluster).}

    \label{fig:pipeline}
\end{figure*}

\section{Methodology}

\subsection{Problem Formulation}

In Novel Category Discovery (NCD), we are given a labeled dataset $\mathcal{D}^l = \{(\mathbf{x}_i^l, y_i^l)\}_{i=1}^{N_l}$ containing images from $C_l$ known categories, and an unlabeled dataset $\mathcal{D}^u = \{\mathbf{x}_j^u\}_{j=1}^{N_u}$ containing images from $C_u$ novel categories. The known and novel categories sets are disjoint. The goal is to discover and cluster the novel categories in $\mathcal{D}^u$ by transferring knowledge learned from $\mathcal{D}^l$. Here we assume number of novel category is a known prior \citep{fini2021unified, vaze2022generalized, fei2022xcon}, as estimating number of novel category for NCD or GCD is considered a separate problem \cite{troisemaine2023novel,han2021autonovel}.

Standard NCD methods discover clusters in high-dimensional feature spaces $\mathbb{R}^{d_z}$, producing opaque groupings that offer no insight into \textit{why} samples belong together. 
We argue that interpretability and discovery are complementary: if we can represent images through human-understandable concepts, we can both cluster effectively \emph{and} explain what distinguishes each discovered category. To this end, we introduce \textbf{Explainable Novel Category Discovery (xNCD)}, which performs discovery in an interpretable concept space $\mathcal{C} = \{c_1, c_2, \ldots, c_K\}$, where each $c_j$ corresponds to a human-understandable visual attribute (e.g., ``striped,'' ``four-legged,'' ``metallic''). 

By representing images through $K$ human-understandable concepts rather than $d_z$-dimensional latent features (where $K < d_z$), we constrain the hypothesis space to partitions expressible through interpretable semantic coordinates (formalized in Section \ref{forml-theory-proof}). The key insight is that such concepts generalize across both known and novel categories: we can describe an unknown animal using attributes like ``fur,'' ``tail,'' and ``four legs'' without knowing its species name. This semantic grounding ensures that discovered clusters correspond to coherent conceptual groupings rather than arbitrary feature correlations.

\subsection{Stage 1: Encoder Pretraining on Labeled categories}

 We pretrain a shared encoder $E_\theta: \mathbb{R}^d \rightarrow \mathbb{R}^{d_z}$ on the labeled dataset $\mathcal{D}^l$ to learn discriminative features for known categories. The encoder outputs feature vectors $\mathbf{z} = E_\theta(\mathbf{x}) \in \mathbb{R}^{d_z}$, which are fed to a linear classifier $h: \mathbb{R}^{d_z} \rightarrow \mathbb{R}^{C_l}$ with $C_l$ output neurons. The model is trained using standard cross-entropy loss:
\begin{equation}
\mathcal{L}_{\text{pretrain}} = -\frac{1}{N_l} \sum_{i=1}^{N_l} \sum_{c=1}^{C_l} \mathbf{y}_i^l[c] \log p_c^l(\mathbf{x}_i^l)
\end{equation}
where $p^l = \sigma(h(E_\theta(\mathbf{x}))/\tau)$ is the softmax output with temperature $\tau$, and $\mathbf{y}^l$ is the one-hot encoded label.

\subsection{Stage 2: Concept Layer Learning}

The second stage learns to project encoder features into an interpretable concept space. The challenge is learning this projection \emph{without} manual concept annotations for each image. We address this by leveraging CLIP as a concept supervisor.

\textbf{CLIP as Concept Supervisor.}
For each image $\mathbf{x}_i$ and concept $c_j$, we compute a target activation:
\begin{equation}
P_{ij} = E^{\text{img}}_{\text{CLIP}}(\mathbf{x}_i) \cdot E^{\text{txt}}_{\text{CLIP}}(c_j),
\end{equation}
where $E^{\text{img}}_{\text{CLIP}}$ and $E^{\text{txt}}_{\text{CLIP}}$ are CLIP's image and text encoders. The matrix $\mathbf{P} \in \mathbb{R}^{N \times K}$ provides dense supervision: $P_{ij}$ is high when concept $c_j$ is visually present in image $\mathbf{x}_i$.

\textbf{Concept Projection Layer.} We learn a projection matrix $\mathbf{W}_c \in 
\mathbb{R}^{K \times d_z}$ that maps encoder features to concept activations, 
following the label-free approach of~\citet{oikarinen2023label}:

\begin{equation}
\hat{\mathbf{c}} = \mathbf{W}_c E_\theta(\mathbf{x}) = \mathbf{W}_c \mathbf{z},
\end{equation}
where $\hat{\mathbf{c}} \in \mathbb{R}^K$ represents predicted concept activations. Each row of $\mathbf{W}_c$ acts as a ``concept detector'' that identifies the corresponding visual attribute.

\textbf{Training Objective.} We train $\mathbf{W}_c$ to align predicted concept activations with CLIP's assessments. For concept $j$, let $\mathbf{q}_j = [\hat{c}_j(\mathbf{x}_1), \ldots, \hat{c}_j(\mathbf{x}_N)]^\top \in \mathbb{R}^N$ be the activation pattern across all images. We maximize the similarity between $\mathbf{q}_j$ and the corresponding column of the CLIP matrix $\mathbf{P}_{:,j}$ using a cubed cosine similarity:
\begin{equation}
\label{eq:cubed_cosine}
    \text{sim}(\mathbf{q}_j, \mathbf{P}_{:,j}) = \frac{\bar{\mathbf{q}}_j^3 \cdot \bar{\mathbf{P}}_{:,j}^3}{\|\bar{\mathbf{q}}_j^3\|_2 \|\bar{\mathbf{P}}_{:,j}^3\|_2},
\end{equation}
where $\bar{\mathbf{v}}$ denotes mean-centered and standardized $\mathbf{v}$, and the cube is applied element-wise. Standard cosine similarity weights all samples equally, but for concept alignment, we care most about images where the concept is clearly present. The cubing operation amplifies high-activation samples where CLIP confidently detects the concept, while suppressing noise from ambiguous cases, leading to sharper concept detectors, more details are in Appendix~\ref{app:sim-details}.

The alignment loss over all concepts is:
\begin{equation}
    \mathcal{L}_{\text{align}}(\mathbf{W}_c) = \sum_{j=1}^{K} -\text{sim}(\mathbf{q}_j, \mathbf{P}_{:,j}).
\end{equation}


After training, we retain only concepts achieving 
$\text{sim}(\mathbf{q}_j, \mathbf{P}_{:,j}) > \delta_{\text{align}}$ on validation data, 
filtering out poorly-aligned concepts. We use $\delta_{\text{align}} = 0.35$ in our 
experiments. We also compute per-concept normalization statistics $(\mu_j, \sigma_j)$ 
for use in Stage 3.

\subsection{Stage 3: Concept-based Discovery}

The discovery stage performs unified learning on both labeled and unlabeled 
data, with all computations occurring in concept space rather than the 
original feature space. This is the key architectural modification that 
enables interpretability.




\textbf{Architecture.} Our architecture (Figure \ref{fig:architecture}) routes all predictions through the concept bottleneck layer. Given an input image $\mathbf{x}$, we compute:
\begin{align}
\mathbf{z} &= E_\theta(\mathbf{x}) \in \mathbb{R}^{d_z}, \\
\tilde{\mathbf{c}} &= \frac{\mathbf{W}_c \mathbf{z} - \boldsymbol{\mu}}{\boldsymbol{\sigma}} \in \mathbb{R}^K,
\end{align}
where $\boldsymbol{\mu}, \boldsymbol{\sigma} \in \mathbb{R}^K$ are per-concept normalization statistics from Stage 2. 
The normalized concept activations $\tilde{\mathbf{c}}$ are processed by two types of classification heads operating in concept space. A labeled head $h:\mathbb{R}^K \rightarrow \mathbb{R}^{C_l}$, implemented as a linear classifier, is responsible for predicting known categories. In parallel, a set of unlabeled heads ${g_n}_{n=1}^N$, each mapping $\mathbb{R}^K \rightarrow \mathbb{R}^{C_u}$ and implemented as an MLP-based clustering module, is used to discover novel categories.



We employ multiple clustering heads to prevent convergence to local optima during discovery. During training, for each batch, we iterate over clustering heads and compute logits by concatenating the labeled head with each clustering head independently:
\begin{equation}
\label{eq:head-eq}
\mathbf{l}_n = [h(\tilde{\mathbf{c}}), g_n(\tilde{\mathbf{c}})] \in \mathbb{R}^{C_l+C_u}, \quad n = 1, \ldots, N.
\end{equation}
The logits are passed through a shared softmax over all $C = C_l + C_u$ categories, enabling joint reasoning about known and novel categories. At test time, we use the head with lowest training loss in the final epoch.

\textbf{Training Objective.} We adopt multi-view self-labeling with swapped prediction~\citep{caron2020unsupervised, 
fini2021unified}, where pseudo-labels from one augmented view supervise the other, 
encouraging transformation-invariant assignments. Crucially, all computations occur in concept space: pseudo-labels are computed from concept-space logits, ensuring that cluster assignments emerge from semantic concept similarities rather than opaque feature correlations.

The unified cross-entropy loss treats labeled supervision and pseudo-labels 
homogeneously:
\begin{equation}
    \mathcal{L}_{\text{discovery}} = 
    \mathbb{E}_{(\mathbf{v}_1, \mathbf{v}_2)}
    [\ell(\mathbf{l}_1, \hat{\mathbf{y}}_2) + \ell(\mathbf{l}_2, \hat{\mathbf{y}}_1)]
\end{equation}
where $\mathbf{l}_1, \mathbf{l}_2$ are concept-space logits from views $\mathbf{v}_1, 
\mathbf{v}_2$, and $\hat{\mathbf{y}}_1, \hat{\mathbf{y}}_2$ are their swapped targets. 
Labeled targets are ground-truth labels padded with zeros; unlabeled pseudo-labels 
are generated via Sinkhorn-Knopp~\citep{cuturi2013sinkhorn} applied to concept-space 
logits (see Appendix~\ref{app:pseudolabels}). Each discovered cluster $i$ is 
characterized by a semantic signature of its top-$r$ activated and deactivated 
concepts ($r{=}5$); formal definitions are in Appendix~\ref{app:explanations}.

\section{Concept Bottlenecks as Hypothesis Space Restriction for Novel Category Discovery}
\label{forml-theory-proof}
We formalize explainable novel category discovery as a problem of learning a classifier whose induced partition over unlabeled data is constrained to admit a meaningful semantic interpretation. Let $\mathcal{X}\subset\mathbb{R}^d$ denote the input space, and let $C = C_l + C_u$ be the total number of categories, consisting of $C_l$ known categories with labels and $C_u$ unknown (novel) categories to be discovered from unlabeled data. In this setting, a learned predictor implicitly defines a partition of the unlabeled dataset $\mathcal{D}_u$ through its decision rule, typically via pseudo-labeling or clustering assignments. 

Most existing novel category discovery (NCD) methods operate in a learned feature space and construct predictors of the form
\begin{equation}
f_{\text{feat}}(x) = \mathrm{softmax}(W \,\phi_\theta(x)),
\end{equation}
where $\phi_\theta:\mathcal{X}\rightarrow\mathbb{R}^{d_z}$ is a learned feature extractor (e.g., a deep convolutional encoder) and $W\in\mathbb{R}^{C\times d_z}$ is a linear classifier whose rows correspond to known and novel category logits. We denote the corresponding hypothesis class by
\begin{equation}
\mathcal{H}_{\text{feat}} = \{x \mapsto \mathrm{softmax}(W\,\phi_\theta(x)) : \theta, W\}.
\end{equation}
This hypothesis class is highly expressive, enabling strong empirical performance in separating known and unknown categories. However, its expressiveness also permits a vast family of decision boundaries and induced partitions that depend on arbitrary geometric correlations in feature space. As a result, clusters discovered under $\mathcal{H}_{\text{feat}}$ need not correspond to coherent or human-understandable semantic attributes, limiting interpretability and trust in open-world deployments.

\paragraph{Concept-Space Formulation.}
Our approach introduces an explicit semantic bottleneck that mediates all predictions through a structured concept space. Specifically, we define a concept map $c_\psi:\mathcal{X}\rightarrow\mathbb{R}^K$ with $K < d_z$, given by
\begin{equation}
c_\psi(x) = \tilde{c}(x) = \frac{W_c \phi_\theta(x) - \mu}{\sigma},
\end{equation}
where $W_c\in\mathbb{R}^{K\times d_z}$ is a learnable concept projection matrix, and $(\mu,\sigma)$ are per-concept normalization statistics estimated during concept training. Each coordinate of $c_\psi(x)$ corresponds to the activation of a human-understandable semantic concept (e.g., visual attributes or contextual properties).

Classification and novel category discovery are then performed entirely in concept space:
\begin{equation}
f_{\text{concept}}(x) = \mathrm{softmax}(A\,c_\psi(x)),
\end{equation}
where $A\in\mathbb{R}^{C\times K}$ aggregates both the labeled-category head and the novel-category discovery head(s). The resulting hypothesis class is
\begin{equation}
\mathcal{H}_{\text{concept}} = \{x \mapsto \mathrm{softmax}(A\,c_\psi(x)) : \theta,\psi,A\}.
\end{equation}
By construction, every prediction and pseudo-label assignment is a function of explicit concept activations, ensuring that the induced partition over $\mathcal{D}_u$ is mediated by semantically interpretable coordinates.

\paragraph{Hypothesis Space Restriction.}
We now show that introducing a concept bottleneck does not merely reparameterize feature-space NCD, but instead induces a strict restriction on the underlying hypothesis space. This restriction eliminates a large family of decision rules that rely on arbitrary feature correlations and enforces that all predictions factor through a low-dimensional, semantically structured subspace.

Recall that feature-space NCD models admit predictors of the form
\[
f_{\text{feat}}(x) = \mathrm{softmax}(W\phi_\theta(x)),
\]
where $W\in\mathbb{R}^{C\times d_z}$ is unconstrained. In contrast, concept-mediated models enforce a two-stage factorization
\[
\phi_\theta(x) \;\longrightarrow\; c_\psi(x) \;\longrightarrow\; f_{\text{concept}}(x),
\]
where all category logits are linear combinations of concept activations. We formalize this distinction below.

\begin{proposition}
\label{prop:subset}
Assume $c_\psi(x)$ is linear in $\phi_\theta(x)$, i.e., $c_\psi(x)=W_c\phi_\theta(x)$ (ignoring normalization for clarity). Then
\begin{equation}
\mathcal{H}_{\text{concept}} \subseteq \mathcal{H}_{\text{feat}}.
\end{equation}
Moreover, if $K < \min\{C,d_z\}$, the inclusion is strict.
\end{proposition}

\noindent\textbf{Proof.}
For any predictor $f_{\text{concept}}\in\mathcal{H}_{\text{concept}}$, the logits satisfy
\[
A\,c_\psi(x) = A W_c \phi_\theta(x) = W' \phi_\theta(x),
\]
where $W' = A W_c \in \mathbb{R}^{C\times d_z}$. Therefore, $f_{\text{concept}}$ can be written as a feature-space classifier and belongs to $\mathcal{H}_{\text{feat}}$, establishing $\mathcal{H}_{\text{concept}} \subseteq \mathcal{H}_{\text{feat}}$.

However, $W'$ is constrained to factor through the $K$-dimensional concept space. In particular, $\mathrm{rank}(W') \le K$, since it is the product of a $C\times K$ matrix and a $K\times d_z$ matrix. When $K < \min\{C,d_z\}$, there exist weight matrices in $\mathbb{R}^{C\times d_z}$ whose rank exceeds $K$ and which therefore cannot be expressed in this factorized form. Consequently, these classifiers are excluded from $\mathcal{H}_{\text{concept}}$, implying strict containment. \hfill$\square$

Beyond dimensional restriction, our framework further constrains 
$\mathcal{H}_{\text{concept}}$ by enforcing alignment between concept 
activations and vision-language similarity priors during training. This 
alignment restricts the admissible concept maps $c_\psi$ to a semantically 
anchored subset, preventing arbitrary rotations or drift of the bottleneck 
that could otherwise occur under unsupervised discovery objectives.

\begin{remark}
When $K < \min\{C, d_z\}$, the rank constraint eliminates all 
$W' \in \mathbb{R}^{C \times d_z}$ with $\mathrm{rank}(W') > K$, 
strictly reducing the hypothesis class. In our experimental settings 
$K > C$ for all datasets (CIFAR-10: $K{=}131$, $C{=}10$; CIFAR-100: 
$K{=}422$, $C{=}100$; CUB-200: $K{=}453$, $C{=}200$), so 
$\mathrm{rank}(W') \leq C < K$ holds trivially and the rank argument 
provides no additional restriction. CLIP alignment then constitutes 
the sole operative interpretability guarantee, constraining admissible 
concept maps to a semantically anchored subset of 
$\mathbb{R}^{K \times d_z}$ independently of the $K$-vs-$C$ relationship.
\end{remark}

Table~\ref{tab:concept_ablation} (Appendix~\ref{app:ablation}) provides 
empirical evidence of this trade-off: reducing concept vocabulary size $K$ 
degrades clustering accuracy, confirming that hypothesis space restriction 
and discovery performance are directly coupled.

\section{Experiments and Results}
\label{sec:exp-res}
\subsection{Experimental Setup}

We evaluate xNCD on CIFAR-10, CIFAR-100~\citep{krizhevsky2009learning}, and 
CUB-200~\citep{wah2011caltech}, split into disjoint labeled and unlabeled sets 
following standard NCD protocols (5+5, 80+20, and 170+30 respectively). We use 
ResNet-18 for CIFAR-10/100 and ResNet-50 for CUB-200; full training details are 
in Appendix~\ref{app:impl-details}. We report labeled accuracy, novel category 
clustering accuracy via Hungarian matching~\citep{kuhn1955hungarian}, overall 
accuracy, NMI, and ARI, under both task-aware and task-agnostic protocols. The 
task-agnostic setting, which jointly discriminates over all $C_l + C_u$ categories 
without oracle task identity, is directly comparable to the GCD evaluation 
protocol~\citep{vaze2022generalized}. Dataset statistics and concept generation 
details are in Appendix~\ref{app:dataset-stats} and~\ref{app:concept-generation}.

\subsection{Clustering Performance}

\begin{figure} [htb]
    \centering
    \includegraphics[width=0.6\linewidth]{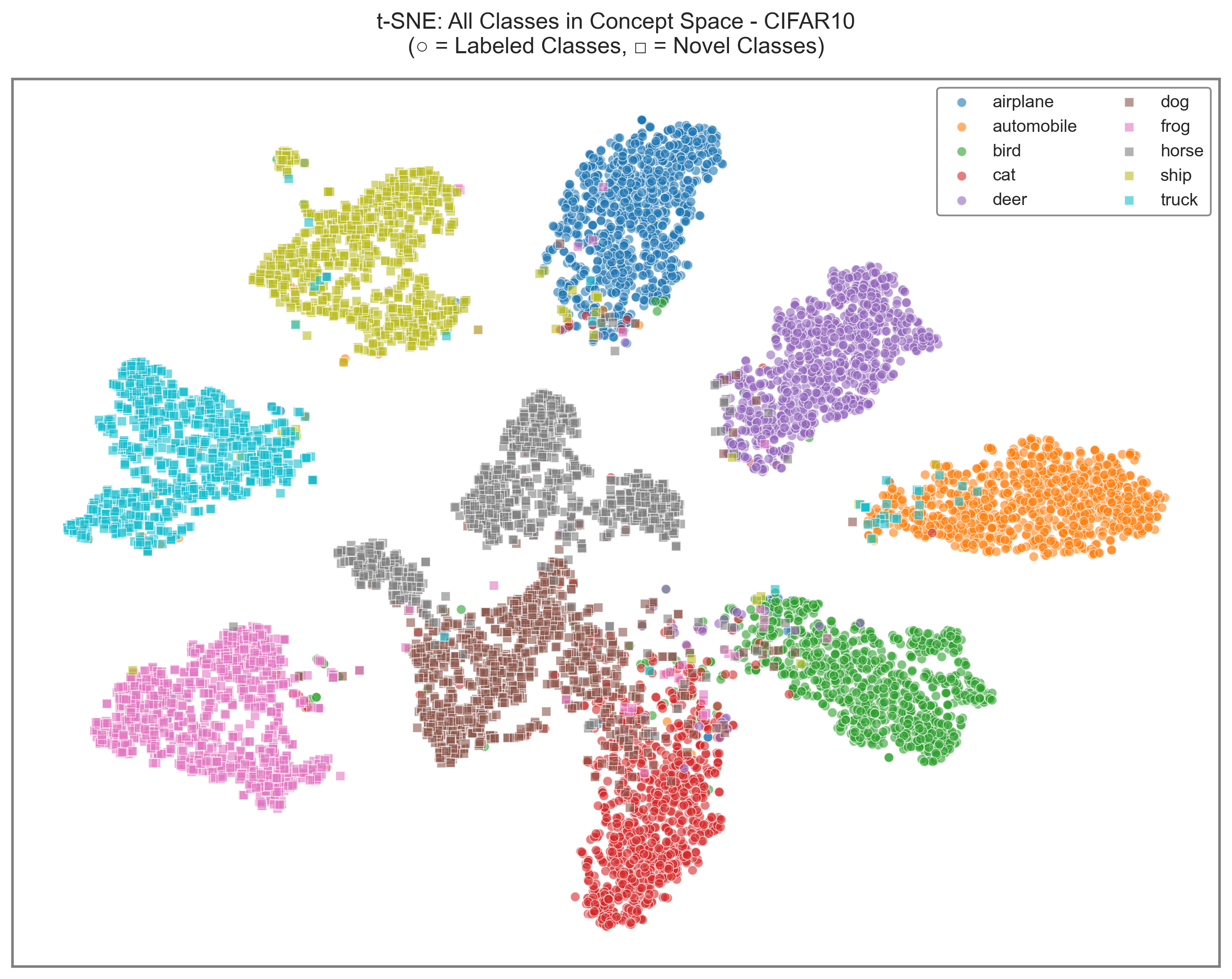}
 \caption{t-SNE visualization of concept space on CIFAR-10. Clear separation between 
all 10 categories (5 labeled, 5 novel) demonstrates that concept space preserves 
discriminative structure. Semantically similar categories (e.g., dog/cat, ship/airplane) 
cluster near each other.}

    \label{fig:tsne}
\end{figure}

Table~\ref{tab:clustering_results} presents the quantitative results of xNCD across three benchmarks under both evaluation protocols.

\textbf{CIFAR-10.} xNCD achieves 95.28\% overall under task-aware and 92.63\% under 
task-agnostic evaluation (Table~\ref{tab:clustering_results}). The modest drop between protocols 
indicates effective separation of known from novel categories. Figure~\ref{fig:tsne} 
shows clear concept-space separation across all 10 categories, with semantically 
similar categories (e.g., dog/cat, ship/airplane) clustering near each other.

\begin{table}[htb]
\centering

\caption{Clustering performance on CIFAR-10, CIFAR-100, and CUB-200 following the split mentioned in Table~\ref{tab:datasets}.}
\label{tab:clustering_results}
\resizebox{0.7\columnwidth}{!}{%
\begin{tabular}{llccc}
\toprule
\textbf{Dataset} & \textbf{Evaluation} & \textbf{Lab (\%)} & \textbf{Unlab (\%)} & \textbf{All (\%)} \\
\midrule
\multirow{2}{*}{\textbf{CIFAR-10}} & Task-aware & 97.46$\pm$0.2 & 93.10$\pm$0.4 & 95.28$\pm$0.3 \\
                                    & Task-agnostic & 94.36$\pm$0.3 & 90.90$\pm$0.5 & 92.63$\pm$0.4 \\
\midrule
\multirow{2}{*}{\textbf{CIFAR-100}} & Task-aware & 79.25$\pm$0.5 & 84.10$\pm$0.7 & 80.22$\pm$0.5 \\
                                     & Task-agnostic & 76.78$\pm$0.4 & 75.15$\pm$0.6 & 76.45$\pm$0.5 \\
\midrule
\multirow{2}{*}{\textbf{CUB-200}} & Task-aware & 70.88$\pm$0.3 & 50.85$\pm$0.6 & 67.88$\pm$0.4 \\
                                   & Task-agnostic & 68.30$\pm$0.4 & 50.28$\pm$0.7 & 65.59$\pm$0.5 \\
\bottomrule
\end{tabular}%
}
\end{table}

\textbf{CIFAR-100.} xNCD achieves 80.22\% overall (task-aware) and 76.45\% 
(task-agnostic). Notably, novel class accuracy (84.10\%) exceeds labeled accuracy 
(79.25\%) under task-aware evaluation, suggesting the concept bottleneck provides 
strong inductive bias for semantically coherent clustering.

\textbf{CUB-200.} xNCD achieves 67.88\% overall (task-aware) and 65.59\% 
(task-agnostic). The lower novel category accuracy (50.85\%) reflects the challenge 
of fine-grained distinction, where subtle visual differences may not be fully captured 
by a general concept vocabulary.



We also evaluate clustering quality using NMI and ARI metrics. xNCD achieves NMI of 
86.34\% and ARI of 85.11\% on CIFAR-10, indicating semantically coherent clusters. 
Detailed ablation studies on concept vocabulary size and clustering metrics are provided 
in Appendix~\ref{app:ablation}.

\subsection{Comparison with State-of-the-Art}

\begin{table}[ht]
\centering
\caption{Comparison with state-of-the-art methods on CIFAR-10 and CIFAR-100 using task-agnostic evaluation. We report accuracy on labeled categories (Lab), clustering accuracy on novel categories (Unlab), and overall accuracy (All). xNCD is the only method providing human-interpretable explanations for discovered clusters.}
\label{tab:comparison}
\resizebox{0.8\columnwidth}{!}{%
\begin{tabular}{lccccccc}
\toprule
\multirow{2}{*}{Method} & \multicolumn{3}{c}{CIFAR-10} & \multicolumn{3}{c}{CIFAR-100} & \multirow{2}{*}{Interpretable} \\
\cmidrule(lr){2-4} \cmidrule(lr){5-7}
& Lab & Unlab & All & Lab & Unlab & All & \\
\midrule
KCL~\citep{hsu2017learning} & 79.4 & 60.1 & 69.8 & 23.4 & 29.4 & 24.6 & \ding{55} \\
MCL~\citep{hsu2019multi} & 81.4 & 64.8 & 73.1 & 18.2 & 18.0 & 18.2 & \ding{55} \\
DTC~\citep{han2019learning} & 58.7 & 78.6 & 68.7 & 47.6 & 49.1 & 47.9 & \ding{55} \\
RS+~\citep{han2020automatically} & 90.6 & 88.8 & 89.7 & 71.2 & 56.8 & 68.3 & \ding{55} \\
UNO~\citep{fini2021unified} & 93.5 & 93.3 & 93.4 & 73.2 & 73.1 & 73.2 & \ding{55} \\
SNCD~\citep{wang2024semantic} & 95.8 & 92.7 & 94.3 & 79.9 & 79.2 & 79.5 & \ding{55} \\
GCD~\citep{vaze2022generalized} & 97.9 & 88.5 & 91.5 & 76.2 & 66.5 & 73.0 & \ding{55} \\
\midrule
\textbf{xNCD (Ours)} & \textbf{94.36} & \textbf{90.90} & \textbf{92.63} & \textbf{76.78} & \textbf{75.15} & \textbf{76.45} & \ding{51} \\
\bottomrule
\end{tabular}%
}
\end{table}

Table~\ref{tab:comparison} compares xNCD against KCL~\citep{hsu2017learning}, 
MCL~\citep{hsu2019multi}, DTC~\citep{han2019learning}, RS+~\citep{han2020automatically}, 
UNO~\citep{fini2021unified}, SNCD~\citep{wang2024semantic}, and GCD~\citep{vaze2022generalized} 
under task-agnostic evaluation. On CIFAR-10, xNCD achieves 92.63\% overall accuracy, 
competitive with UNO (93.4\%) and GCD (91.5\%), trailing SNCD (94.3\%) by 1.67\%. 
On CIFAR-100, xNCD outperforms UNO (73.2\%) and GCD (73.0\%) by over 3 points, 
remaining competitive with SNCD (79.5\%). Critically, xNCD is the only compared 
method providing human-interpretable cluster- and instance-level explanations a 
capability that no accuracy metric can capture, yet is essential for validating 
discovered categories in scientific and safety-critical applications. The modest 
accuracy gap relative to non-interpretable methods confirms that routing discovery 
through a concept bottleneck does not significantly compromise clustering quality. 
Extended discussion is in Appendix~\ref{app:discussion}.


\begin{figure*}[htb]
    \centering
    \includegraphics[width=0.49\linewidth]{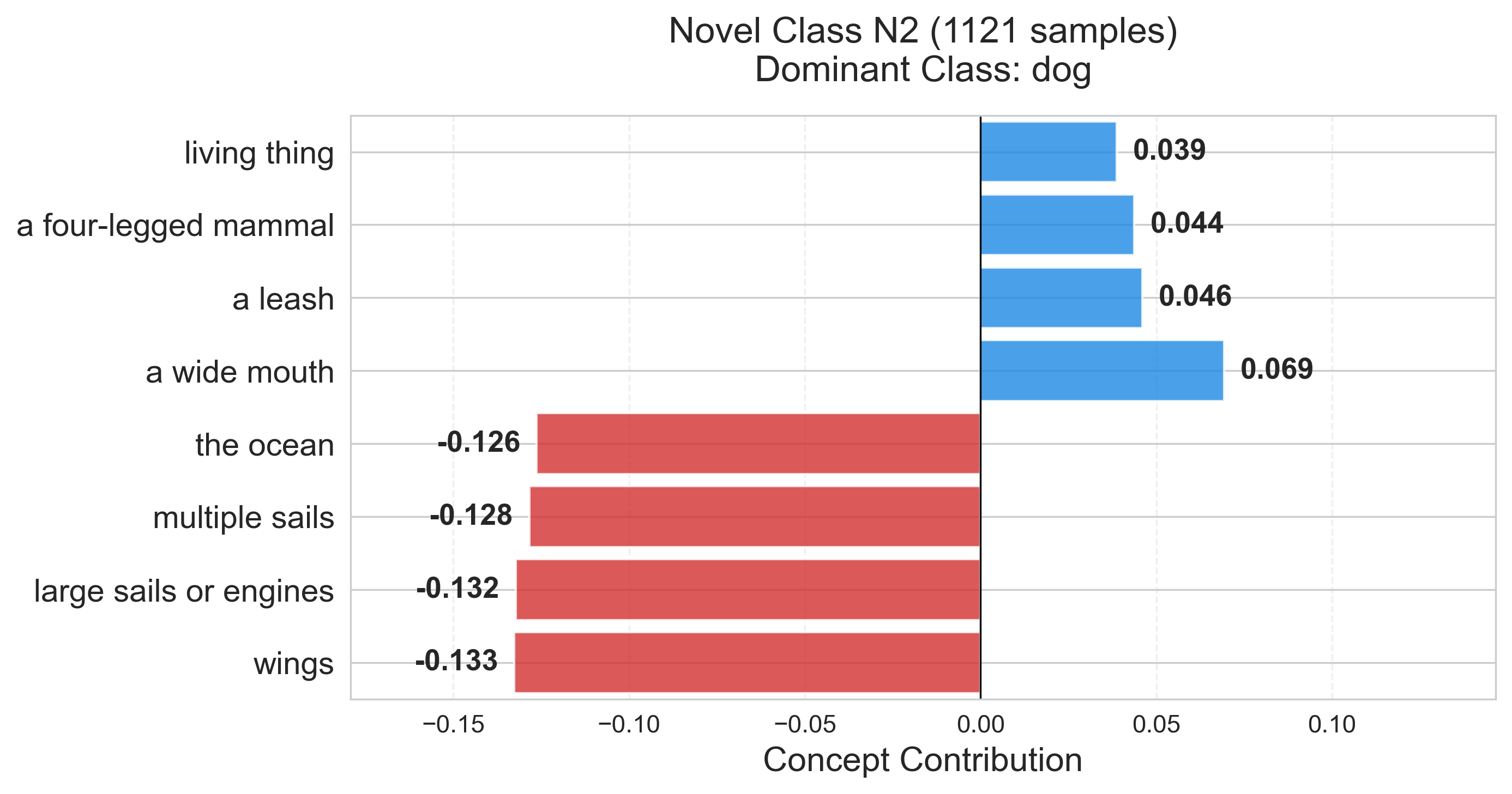}
    \includegraphics[width=0.49\linewidth]{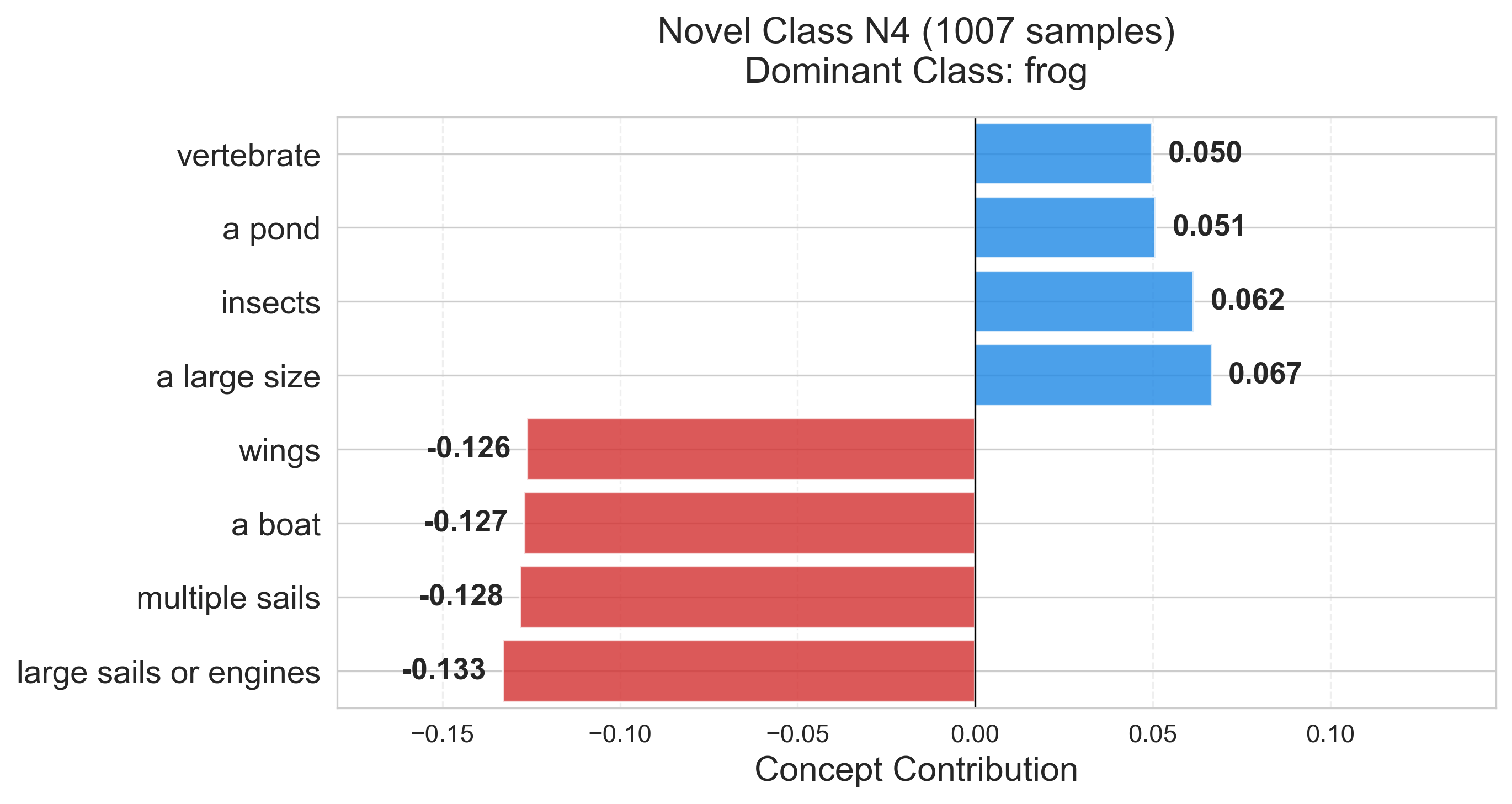}

\caption{Concept activation profiles for discovered novel categories on 
CIFAR-10. The dog cluster (left) activates \textit{living thing, four-legged 
mammal, wide mouth} while deactivating \textit{ocean, wings}, correctly 
distinguishing terrestrial mammals from vehicles and birds. The frog 
cluster (right) activates \textit{vertebrate, pond} while deactivating 
boat-related concepts.}
    \label{fig:concept_profiles}
\end{figure*}

\begin{figure*} [htb]
    \centering

   \includegraphics[width=0.25\linewidth]{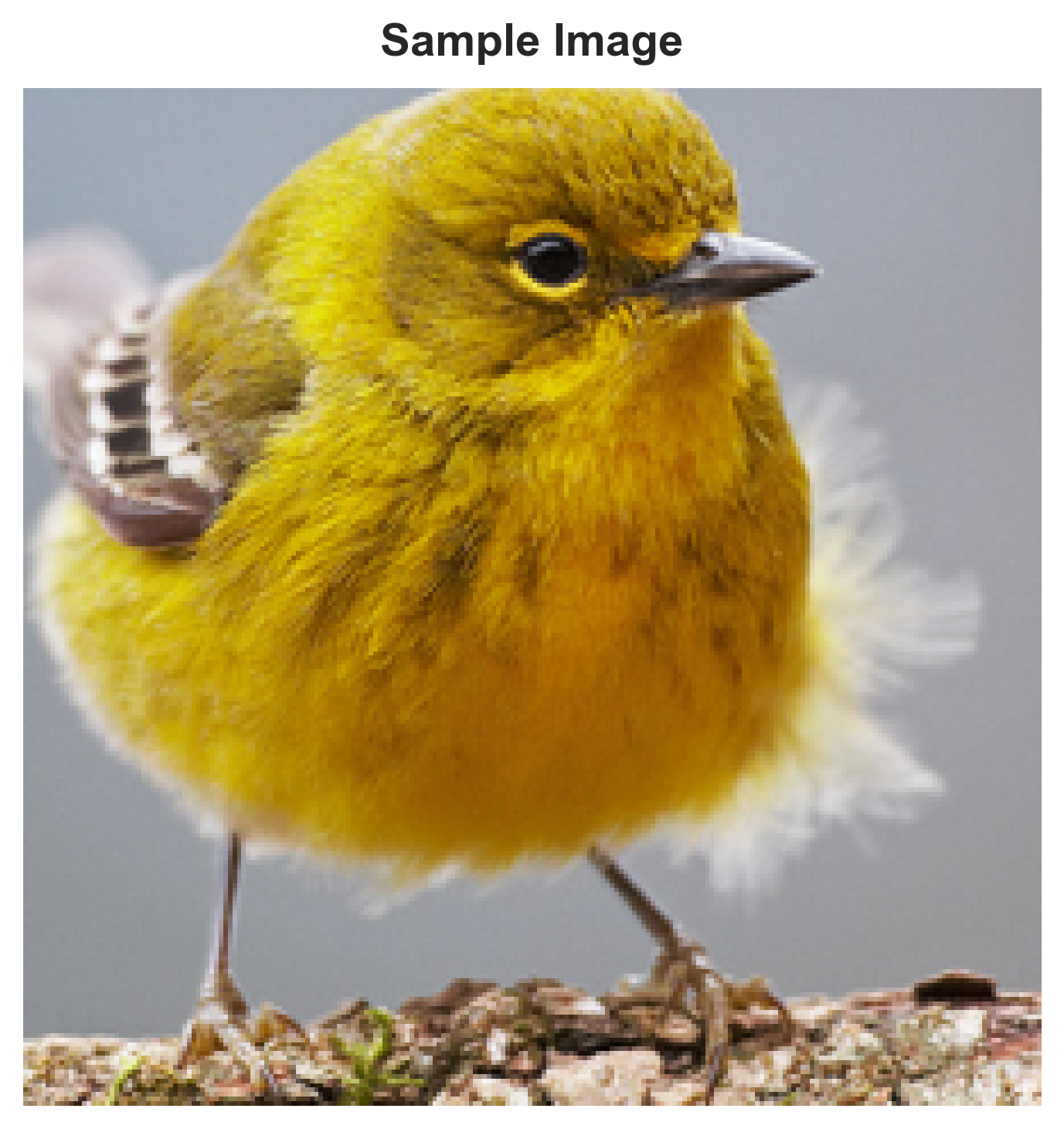}
    \includegraphics[width=0.52\linewidth]{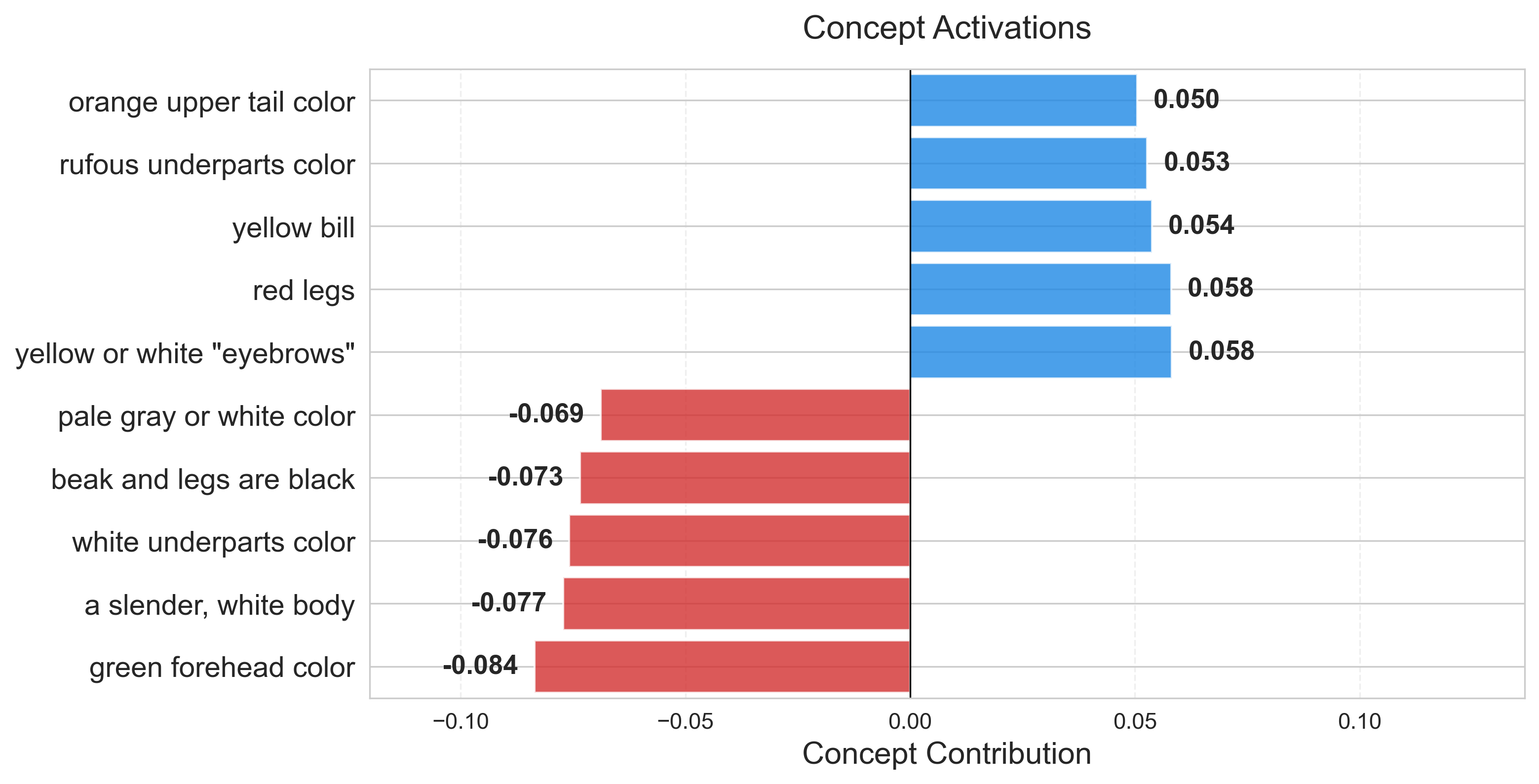}

\caption{Instance-level concept attribution for a CUB-200 bird image. 
Activated concepts (blue) such as \textit{yellow bill} and \textit{yellow eyeball}
provide positive evidence, while suppressed concepts (red) such as 
\textit{green forehead} and \textit{pale gray color} indicate absent attributes. 
Such explanations enable domain experts to validate discovered categories 
by understanding not just \textit{what} the model predicts but 
\textit{why}.}

    \label{fig:f-instance-level}
\end{figure*}

\subsection{Interpretability Analysis}

\textbf{Cluster-Level Concept Signatures.} Figure~\ref{fig:concept_profiles} visualizes 
concept profiles for discovered CIFAR-10 clusters. The dog cluster is characterized by 
strong positive activations for \textit{living thing}, \textit{four-legged mammal}, 
\textit{wide mouth} while strongly deactivating \textit{ocean}, \textit{wings}, 
\textit{sails}. These semantic signatures are highly informative: positive activations 
describe what the cluster \textit{is}, while negative activations clarify what it is 
\textit{not}. The dog cluster's deactivation of ocean and wing concepts confirms the 
model correctly distinguishes terrestrial mammals from vehicles and birds.

\textbf{Instance-Level Explanations.} Figure~\ref{fig:f-instance-level} shows 
explanations for a CUB-200 bird image. The model highlights \textit{yellow bill}, 
\textit{rufous underparts}, \textit{orange upper tail} as positive evidence while 
suppressing \textit{green forehead}, \textit{white underparts}. These instance-level 
explanations enable practitioners to understand not only \textit{what} the model predicts 
but \textit{why} a capability critical for deploying NCD in scientific domains where 
discovered categories must be validated by domain experts. 
Additional visualizations are in 
Appendix~\ref{app:vis}.

\section{Conclusion}

We presented Explainable Novel Category Discovery (xNCD), a framework that performs novel category discovery directly in a semantic concept space, replacing opaque feature-space clustering with intrinsically interpretable discovery. By routing all predictions and pseudo-labels through a concept bottleneck learned via label-free alignment with vision–language priors, xNCD produces novel categories that admit concise, human-understandable concept signatures.
We provide a theoretical analysis showing that concept bottlenecks impose a strict restriction on the hypothesis space of discovery models, eliminating semantically entangled partitions while preserving expressive power. Empirical results on CIFAR-10, CIFAR-100, and CUB-200 demonstrate that xNCD achieves competitive or improved discovery performance relative to state-of-the-art methods, while uniquely enabling transparent cluster- and instance-level explanations.

\bibliographystyle{unsrtnat}
\bibliography{bibtext}

\newpage
\appendix

\section{Concept Set Generation.} 
\label{app:concept-generation}
For all dataset we generated concepts following~\citet{oikarinen2023label}, we generate initial concept sets using GPT~\citep{brown2020language} with three prompt types: (i) important features for recognizing each category, (ii) things commonly seen around each category, and (iii) superclasses. We apply filtering to remove: concepts longer than 30 characters, concepts too similar to category names (cosine similarity $>0.85$), duplicate concepts (similarity $>0.9$), concepts not present in training data, and concepts with poor projection accuracy (CLIP similarity $<0.35$). The final concept counts are shown in Table~\ref{tab:datasets}.

CUB-200 dataset comes with expert annotated 312 concepts for all images and we generated 369 more concepts following the procedure described above and combined them together to have a total of 671 concepts. We have ran experiment on how our xNCD works with different concept sizes more details are in \ref{app:cp-size}

\section{Algorithm}
\label{app:algo}

\subsection{Stage 1: Encoder Pretraining}
\begin{algorithmic}[1]
\FOR{epoch $= 1$ to $200$}
    \FOR{$(\mathbf{x}, y) \in \mathcal{D}^l$}
        \STATE $\mathbf{z} \gets E_\theta(\mathbf{x})$
        \STATE $\mathbf{p}^l \gets \sigma(h(\mathbf{z})/\tau)$
        \STATE $\mathcal{L} \gets -\sum_{c} y_c \log p_c^l$
        \STATE Update $\theta$ via SGD
    \ENDFOR
\ENDFOR
\STATE Save pretrained encoder $E_\theta$
\end{algorithmic}

\subsection{Stage 2: Concept Layer Learning}
\begin{algorithmic}[1]
\STATE Load pretrained encoder $E_\theta$ (frozen)
\STATE Extract features: $\mathbf{z}_i = E_\theta(\mathbf{x}_i)$ for all $\mathbf{x}_i \in \mathcal{D}^{\text{all}}$
\STATE Compute CLIP matrix: $P_{ij} = E_{\text{CLIP}}^{\text{img}}(\mathbf{x}_i) \cdot E_{\text{CLIP}}^{\text{txt}}(c_j)$ for all images
\STATE Initialize $\mathbf{W}_c$ randomly
\FOR{training iterations with early stopping}
    \STATE $\mathcal{L} \gets \sum_{j=1}^{K} -\text{sim}(c_j, \mathbf{q}_j)$
    \STATE Update $\mathbf{W}_c$ via Adam optimizer
\ENDFOR
\STATE Filter concepts: Keep only $c_j$ where $\text{sim}(c_j, \mathbf{q}_j) \geq 0.35$
\STATE Update $K \gets K - \Delta$ where $\Delta$ is number of filtered concepts
\STATE Save trained $\mathbf{W}_c$, concept normalization statistics, and final concept list
\end{algorithmic}

\subsection{Stage 3: Concept-based Discovery}
\begin{algorithmic}[1]
\STATE Load pretrained $E_\theta$ and trained $\mathbf{W}_c$
\STATE Initialize unlabeled heads $\{g_n\}$ and overclustering heads $\{o_n\}$
\FOR{epoch $= 1$ to $200$}
    \FOR{batch $\mathcal{B} = \mathcal{B}^l \cup \mathcal{B}^u$}
        \STATE Generate views: $\{\mathbf{v}_1, \mathbf{v}_2\}$ for each $\mathbf{x} \in \mathcal{B}$
        \STATE Extract concepts: $\hat{\mathbf{c}}_1 = \mathbf{W}_c E_\theta(\mathbf{v}_1)$, $\hat{\mathbf{c}}_2 = \mathbf{W}_c E_\theta(\mathbf{v}_2)$
        \STATE Normalize: $\tilde{\mathbf{c}}_1 = (\hat{\mathbf{c}}_1 - \boldsymbol{\mu}) / \boldsymbol{\sigma}$, $\tilde{\mathbf{c}}_2 = (\hat{\mathbf{c}}_2 - \boldsymbol{\mu}) / \boldsymbol{\sigma}$
        \STATE Compute logits: $\mathbf{l}_1 = [h(\tilde{\mathbf{c}}_1), g(\tilde{\mathbf{c}}_1)]$, $\mathbf{l}_2 = [h(\tilde{\mathbf{c}}_2), g(\tilde{\mathbf{c}}_2)]$
        \STATE \textbf{For labeled samples:} $\mathbf{y}_1 = \mathbf{y}_2 = [\mathbf{y}^l, \mathbf{0}_{C_u}]$
        \STATE \textbf{For unlabeled samples:} Generate $\hat{\mathbf{y}}_1, \hat{\mathbf{y}}_2$ via Sinkhorn-Knopp on $g(\tilde{\mathbf{c}})$
        \STATE Compute loss: 
    $\mathcal{L}_{\text{discovery}} = \mathbb{E}_{(\mathbf{v}_1, \mathbf{v}_2)} 
    \left[ \ell(\tilde{\mathbf{c}}_1, \hat{\mathbf{y}}_2) + \ell(\tilde{\mathbf{c}}_2, \hat{\mathbf{y}}_1) \right]$
        \STATE Update parameters via SGD with momentum
    \ENDFOR
\ENDFOR
\end{algorithmic}


\section{Similarity Function Details}
\label{app:sim-details}
\subsection{Cubed Cosine Similarity for Concept Alignment}

In Stage 2, we align learned concept activations with CLIP's vision-language similarity priors. We use a cubed cosine similarity function (Equation~\eqref{eq:cubed_cosine}):
\[
\text{sim}(\mathbf{q}_j, \mathbf{P}_{:,j}) = \frac{\bar{\mathbf{q}}_j^3 \cdot \bar{\mathbf{P}}_{:,j}^3}{\|\bar{\mathbf{q}}_j^3\|_2 \|\bar{\mathbf{P}}_{:,j}^3\|_2}
\]
where $\bar{\mathbf{v}}$ denotes mean-centered and standardized $\mathbf{v}$, and the cube is applied element-wise.

\subsection{Motivation}

Standard cosine similarity weights all samples equally when measuring correlation between activation patterns. However, for concept alignment in novel category discovery, this uniform weighting is suboptimal for two reasons:

\textbf{Sparse concept presence.} Most concepts are present in only a small fraction of images. For example, a concept like ``striped'' may be strongly present in a small subset of images while being absent or ambiguous in the remainder. Standard cosine similarity dilutes the signal from these informative samples with noise from the majority.

\textbf{Gradient signal concentration.} During optimization, we want gradients to flow primarily from samples where the concept is clearly present or clearly absent, not from ambiguous cases where CLIP's assessment may be unreliable.

Polynomial amplification addresses both issues by emphasizing extreme values. For a standardized activation $\bar{v}_i \in [-3, 3]$ (approximately), the cubic power preserves sign while amplifying magnitude for $|\bar{v}_i| > 1$, effectively reweighting samples so high-activation images contribute disproportionately to the similarity score~\cite{oikarinen2022clip, oikarinen2023label}.


Concepts with similarity below $\delta_{\text{align}} = 0.35$ are filtered out, 
ensuring that retained concepts are faithfully represented in the concept bottleneck layer. 
Table~\ref{tab:concept_stats} (Appendix \ref{app:ablation}) shows that this approach achieves average 
CLIP similarity between 0.371--0.487 across datasets, with concept retention rates of 
65--92\%, indicating effective concept learning in our NCD setting.

\section{Multi-view Self-labeling and Pseudo-label Generation}
\label{app:pseudolabels}

In this section, we describe how multi-view self-labeling is used to generate pseudo-labels for our unified objective. This approach, originally developed for self-supervised and semi-supervised learning~\citep{caron2020unsupervised, chen2020simple}, enables effective knowledge transfer from labeled to unlabeled data.

\subsection{Multi-view Generation}

Given an image $\mathbf{x}$, we generate two augmented views $\mathbf{v}_1$ and $\mathbf{v}_2$ by applying random transformations consisting of random cropping, color jittering, and horizontal flipping. These augmentations, which have proven effective in contrastive learning~\citep{chen2020simple}, encourage the model to learn transformation-invariant representations.

For labeled samples $(\mathbf{x}, y^l) \in \mathcal{D}^l$, both views are associated with the same ground-truth label, padded with zeros for the novel categories:
\begin{equation}
    \mathbf{y}_1 = \mathbf{y}_2 = [\mathbf{y}^l, \mathbf{0}_{C_u}].
\end{equation}

For unlabeled samples $\mathbf{x} \in \mathcal{D}^u$, we use the \textit{swapped prediction} task~\citep{caron2020unsupervised}: the pseudo-label computed from one view serves as the target for the other view:
\begin{equation}
    \mathbf{y}_1 = [\mathbf{0}_{C_l}, \hat{\mathbf{y}}_2], \quad \mathbf{y}_2 = [\mathbf{0}_{C_l}, \hat{\mathbf{y}}_1].
\end{equation}

This formulation encourages the model to output consistent predictions for different augmentations of the same image. When evaluating each term in the loss, we apply a stop-gradient for the pseudo-label, i.e., the gradient flows only through the prediction branch.

\subsection{Pseudo-label Generation via Optimal Transport}

A na\"ive approach to generate pseudo-labels would be to directly use the softmax outputs of $g(\tilde{\mathbf{c}})$, setting $\hat{\mathbf{y}}_2 = \sigma(g(\tilde{\mathbf{c}}_1)/\tau)$. However, as observed in~\citet{asano2019self}, this may lead to degenerate solutions where the model always predicts the same cluster for any input---in this case, the prediction and pseudo-label become identical and no learning occurs.

To prevent such collapse, we compute pseudo-labels using the Sinkhorn-Knopp algorithm~\citep{cuturi2013sinkhorn}, which adds an entropy regularizer that encourages uniform partition of pseudo-labels across all clusters. For a mini-batch containing $B^u$ unlabeled samples, let $\mathbf{L} = [\mathbf{l}_1, \ldots, \mathbf{l}_{B^u}] \in \mathbb{R}^{C_u \times B^u}$ be the matrix whose columns are the concept-space logits computed by $g(\tilde{\mathbf{c}})$. The pseudo-label matrix $\hat{\mathbf{Y}} = [\hat{\mathbf{y}}_1, \ldots, \hat{\mathbf{y}}_{B^u}]^\top \in \mathbb{R}^{C_u \times B^u}$ is obtained by solving:
\begin{equation}
    \hat{\mathbf{Y}} = \arg\max_{\mathbf{Y} \in \Gamma} \text{Tr}(\mathbf{Y}\mathbf{L}^\top) + \epsilon \, \text{H}(\mathbf{Y}),
\end{equation}
where $\epsilon > 0$ is a hyperparameter, $\text{H}(\mathbf{Y}) = -\sum_{ij} Y_{ij} \log Y_{ij}$ is the entropy function used to ``scatter'' the pseudo-labels, and $\Gamma$ is the transportation polytope defined as:
\begin{equation}
    \Gamma = \left\{ \mathbf{Y} \in \mathbb{R}^{C_u \times B^u}_+ \;\middle|\; \mathbf{Y}\mathbf{1}_{B^u} = \frac{1}{C_u}\mathbf{1}_{C_u}, \;\; \mathbf{Y}^\top\mathbf{1}_{C_u} = \frac{1}{B^u}\mathbf{1}_{B^u} \right\}.
\end{equation}

These constraints enforce that, on average, each cluster is selected $B^u / C_u$ times per batch, preventing any single cluster from dominating. Following~\citet{caron2020unsupervised}, we use 3 Sinkhorn iterations with $\epsilon = 0.05$. The resulting pseudo-labels $\hat{\mathbf{y}}_j \in [0,1]^{C_u}$ are soft assignments; we found that using soft pseudo-labels yields better performance than discretizing them.

\subsection{Concept-Space Pseudo-labels: The Key Difference}

The critical distinction between xNCD and standard NCD methods lies in \textit{where} pseudo-labels are computed. In prior work such as UNO~\citep{fini2021unified}, pseudo-labels are generated from opaque feature representations cluster assignments reflect arbitrary similarities in high-dimensional feature space with no semantic meaning.

In xNCD, pseudo-labels are computed from concept-space logits $g(\tilde{\mathbf{c}})$, where $\tilde{\mathbf{c}} \in \mathbb{R}^K$ represents interpretable concept activations. This design choice has two important consequences:

\begin{enumerate}
    \item \textbf{Interpretable clustering}: Samples are grouped together because they share similar \textit{concept activations}, not because of opaque feature correlations. A cluster of dogs emerges because its members all activate ``four-legged,'' ``furry,'' and ``mammal'' concepts.
    
    \item \textbf{Semantic consistency}: The Sinkhorn-Knopp constraints operate in concept space, encouraging each discovered cluster to have a coherent concept profile rather than arbitrary feature statistics.
\end{enumerate}

This is the architectural modification that enables xNCD to provide human-understandable explanations for discovered categories while maintaining competitive clustering performance.

\section{Formal Definitions of Explanations}
\label{app:explanations}

\subsection{Interpretable Cluster Characterization}

A key advantage of xNCD is that discovered clusters admit human-understandable descriptions through their concept activation profiles.

\subsubsection{Cluster-Level Signatures}

After discovery, each novel category $i \in \{1, \ldots, C_u\}$ is characterized by its mean concept activation vector. Let $S_i = \{\mathbf{x} \in \mathcal{D}^u \mid \hat{y}(\mathbf{x}) = C_l + i\}$ denote samples assigned to cluster $i$. The cluster's concept profile is:
\begin{equation}
\bar{\mathbf{c}}^{(i)} = \frac{1}{|S_i|} \sum_{\mathbf{x} \in S_i} \tilde{\mathbf{c}}(\mathbf{x}) \in \mathbb{R}^K.
\end{equation}
The top-$r$ activated concepts provide a semantic signature:
\begin{equation}
\text{Signature}(i) = \text{Top}_r(\bar{\mathbf{c}}^{(i)}),
\end{equation}
where $\text{Top}_r$ returns indices of the $r$ concepts with highest mean activation. In our experiments, we use $r = 10$ to provide concise yet informative cluster characterizations.

\subsubsection{Instance-Level Explanations}

For a test image $\mathbf{x}$ predicted as novel category $i$, we explain the prediction by identifying which signature concepts are strongly activated:
\begin{equation}
\text{Explanation}(\mathbf{x}, i) = \left\{c_j \mid \tilde{c}_j(\mathbf{x}) > \delta_{\text{explain}}, \; j \in \text{Signature}(i)\right\},
\end{equation}
where $\delta_{\text{explain}}$ is an activation threshold. In our experiments, we use $\delta_{\text{explain}} = 0$ (i.e., include all positively activated signature concepts).

This enables explanations such as: \textit{``This image belongs to novel category 3 because it strongly activates concepts \{four-legged, fur, tail\} which characterize this cluster.''} Such explanations allow domain experts to validate whether discovered categories are semantically meaningful and to identify potential failure modes.

\section{Implementation Details}
\label{app:impl-details}

During pre-training stage 1, we pretrain the encoder for 200 epochs on labeled data using SGD with momentum 0.9 (lr=0.1, cosine annealing, weight decay $10^{-4}$). Temperature $\tau=0.1$ for all softmax layers. In Stage 2, we train the concept projection $W_c$ for 1000 iterations using Adam optimizer (lr=$10^{-3}$) with early stopping (patience=5 validation checks). We use an 80/20 train/validation split. The projection batch size is 50,000 samples. Concepts with CLIP similarity below $\delta_{\text{align}} = 0.35$ are filtered out.

In the discovery phase, we train for 200 epochs using SGD with momentum 0.9 (base lr=0.4, weight decay $1.5 \times 10^{-4}$). We use 5 clustering heads with overclustering factor 3. The Sinkhorn-Knopp algorithm uses 3 iterations with $\epsilon=0.05$. Concept activations are normalized using statistics computed in Stage 2: $\tilde{c} = (W_c z - \mu) / \sigma$, where $\mu, \sigma \in \mathbb{R}^K$ are per-concept mean and standard deviation.

\section{Dataset Statistics}
\label{app:dataset-stats}
\begin{table}[!h]
\centering
\caption{Statistics of the datasets and splits used in our novel category discovery benchmark. Each dataset is split into disjoint labeled (known) and unlabeled (novel) category sets.}
\label{tab:datasets}
\begin{tabular}{@{}llcccc@{}}
\toprule
\multirow{2}{*}{Dataset} & \multicolumn{2}{c}{Labeled} & \multicolumn{2}{c}{Unlabeled} & \multirow{2}{*}{$K$} \\
\cmidrule(lr){2-3} \cmidrule(lr){4-5}
 & Images & categories & Images & categories & \\
\midrule
CIFAR-10 & 25K & 5 & 25K & 5 & 143 \\
CIFAR-100 & 40K & 80 & 10K & 20 & 892 \\
CUB-200 & 8.5K & 170 & 1.4K & 30 & 671 \\
\bottomrule
\end{tabular}
\end{table}

\section{Ablation Study}
\label{app:ablation}
\subsection{Clustering Quality Metrics}
\label{app:cquality-metrics}
We evaluate clustering quality using NMI and ARI (Table~\ref{tab:clustering_metrics}). xNCD achieves strong clustering coherence, with NMI of $86.34\%$ and ARI of $85.11\%$ on CIFAR-10 under task-aware evaluation, indicating that discovered clusters are semantically meaningful.

\begin{table}[ht]
\centering
\caption{Clustering quality metrics on novel categories.}
\label{tab:clustering_metrics}
\small
\begin{tabular}{l|cc|cc}
\toprule
\multirow{2}{*}{Dataset} & \multicolumn{2}{c|}{Task-Aware} & \multicolumn{2}{c}{Task-Agnostic} \\
\cmidrule(lr){2-3} \cmidrule(lr){4-5}
 & NMI & ARI & NMI & ARI \\
\midrule
CIFAR-10 & 86.34$_{\pm0.3}$ & 85.11$_{\pm0.4}$ & 83.22$_{\pm0.4}$ & 83.21$_{\pm0.5}$ \\
CIFAR-100 & 79.00$_{\pm0.5}$ & 70.39$_{\pm0.6}$ & 76.63$_{\pm0.6}$ & 65.36$_{\pm0.7}$ \\
CUB-200 & 63.57$_{\pm0.8}$ & 33.39$_{\pm1.0}$ & 65.73$_{\pm0.9}$ & 34.31$_{\pm1.1}$ \\
\bottomrule
\end{tabular}
\end{table}

\subsection{Concept Set Statistics.} Table~\ref{tab:concept_stats} summarizes concept filtering across datasets. Initial concepts are generated via GPT-3, then pre-filtered to remove duplicates and category-similar concepts. After Stage 2 training, we apply a second filtering step: concepts with learned CLIP
similarity below threshold $\delta_{\text{align}} = 0.35$ are removed, as they indicate poor alignment between the projection layer and CLIP's semantic understanding. This ensures all retained concepts are faithfully represented in the concept bottleneck layer.

\begin{table}[htb]
\centering
\caption{Concept set statistics. \textit{Generated}: after initial filtering (Appendix~\ref{app:concept-generation}). \textit{Final}: after Stage 2 CLIP similarity filtering ($\delta_{\text{align}} = 0.35$). \textit{Avg. Sim.}: average CLIP similarity of retained concepts.}
\label{tab:concept_stats}
\small
\begin{tabular}{lcccc}
\toprule
Dataset & Generated & Final & Avg. Sim. \\
\midrule
CIFAR-10  & 143 & 131 & 0.487 \\
CIFAR-100  & 892 & 422 & 0.371 \\
CUB-200  & 671 & 453 & 0.402 \\
\bottomrule
\end{tabular}
\end{table}

\subsection{Effect of Concept Vocabulary Size}
\label{app:cp-size}
Table~\ref{tab:concept_ablation} examines how concept vocabulary size affects discovery on CUB-200. We evaluate expert-annotated attributes (312 concepts), GPT-generated descriptions (369 concepts), and their combination (671 concepts).

Results demonstrate that larger vocabularies improve performance: the combined set achieves 67.88\%, outperforming smaller vocabularies by approximately 5\%. This suggests that richer concept spaces provide more discriminative dimensions for distinguishing fine-grained novel categories. Additionally, the comparable performance between expert and GPT concepts indicates that concept diversity may be as important as domain-specific quality for novel category discovery.

\begin{table}[htb]
\centering
\caption{Ablation study on concept source for CUB-200 (170+30 split). Expert concepts are domain-specific bird attributes, GPT concepts are LLM-generated descriptions.}
\label{tab:concept_ablation}
\small
\begin{tabular}{lcc}
\toprule
Concept Source & \# Concepts & Accuracy (\%) \\
\midrule
Combined & 671 & 67.88 \\
GPT-generated & 369 & 63.12 \\
Expert-annotated & 312 & 64.48 \\
\bottomrule
\end{tabular}
\end{table}

\section{Discussion}
\label{app:discussion}

\textbf{Summary of Findings.} Our experiments demonstrate that xNCD successfully bridges interpretability and novel category discovery. On standard benchmarks, xNCD achieves clustering accuracy within 1-3\% of state-of-the-art methods (92.63\% vs. UNO's 93.4\% on CIFAR-10) while being the only method to provide human-understandable explanations for discovered categories. Notably, on CIFAR-100 under task-agnostic evaluation, xNCD outperforms UNO (76.45\% vs. 73.2\%), suggesting that concept-based representations may offer advantages for distinguishing between known and novel categories.

\textbf{The Accuracy-Interpretability Trade-off.} Our results reveal a nuanced trade-off between clustering performance and interpretability. Removing the concept bottleneck layer entirely (operating in raw feature space) yields a 1-2\% accuracy improvement, confirming that the concept layer introduces some information bottleneck. However, this cost is modest and is offset by the substantial interpretability benefits: discovered clusters can be characterized through concept profiles, enabling domain expert validation and model debugging.

\textbf{Limitations.} Our approach inherits limitations from both concept bottleneck models 
and novel category discovery. First, the concept vocabulary must be either generated (via LLMs) 
or annotated by domain experts, and our dependency on CLIP may miss domain-specific 
attributes—particularly problematic for specialized domains (e.g., medical imaging) where 
CLIP lacks relevant training data. Second, like most NCD methods, we assume the number of 
novel categories $C_u$ is known a priori; real-world open-world scenarios may contain 
substantially more novel categories than our experimental settings. Finally, while our 
unified objective eliminates self-supervised pretraining, performance still depends on 
semantic similarity between labeled and unlabeled categories for effective knowledge transfer.

\textbf{Broader Impact.} xNCD enables interpretable open-world learning, with potential applications in scientific discovery domains where understanding \textit{why} categories emerge is as important as discovering them. In biodiversity monitoring, ecologists could validate whether discovered species groupings reflect meaningful taxonomic distinctions. In drug discovery, novel compound clusters could be characterized through interpretable molecular properties. We do not foresee negative societal consequences specific to this work.

\section{Visualization}
\label{app:vis}
\subsection{Cluster-wise Explanation}
\begin{figure*} [ht]
    \centering
    \includegraphics[width=0.49\linewidth]{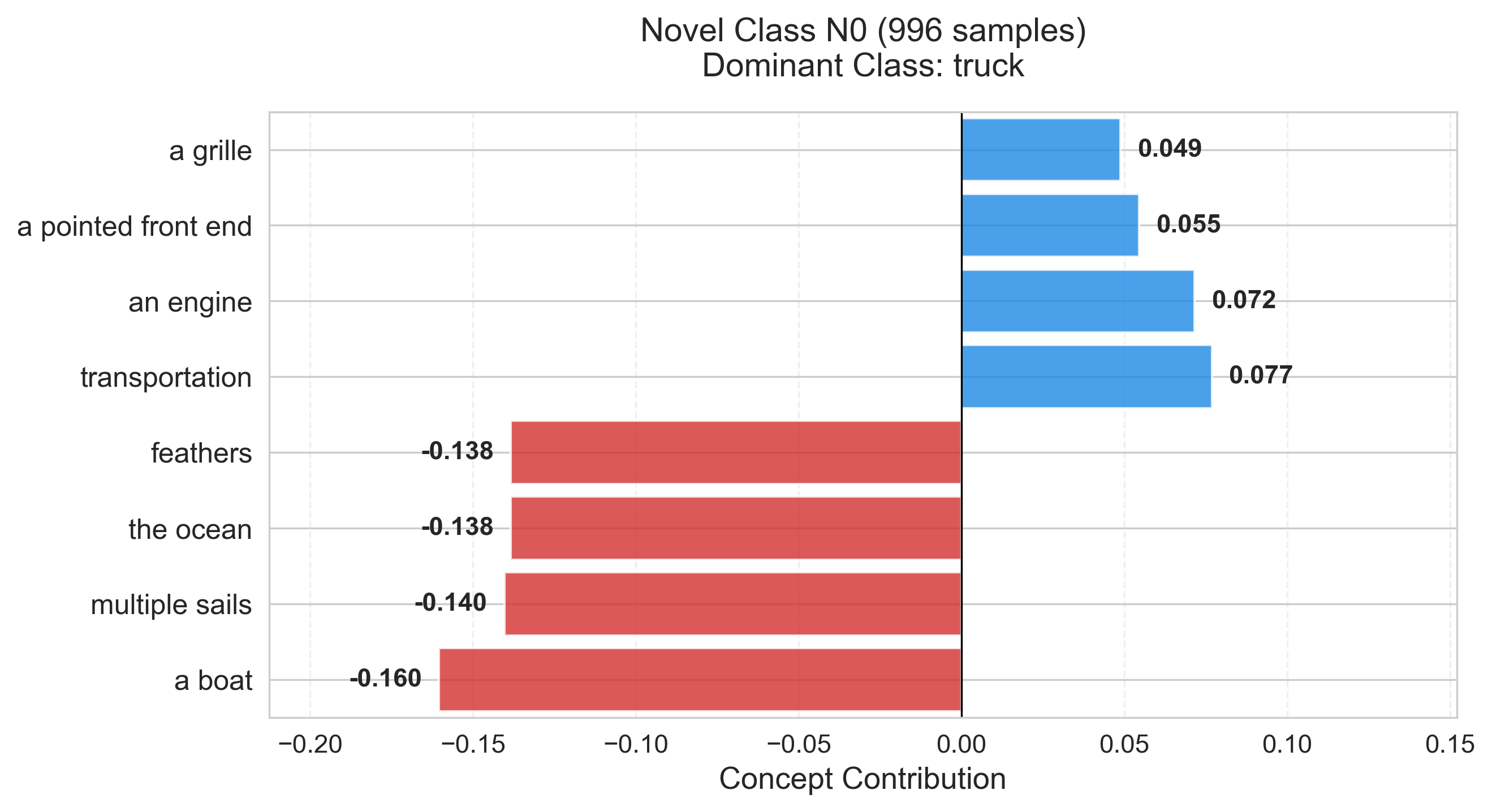}
    \includegraphics[width=0.49\linewidth]{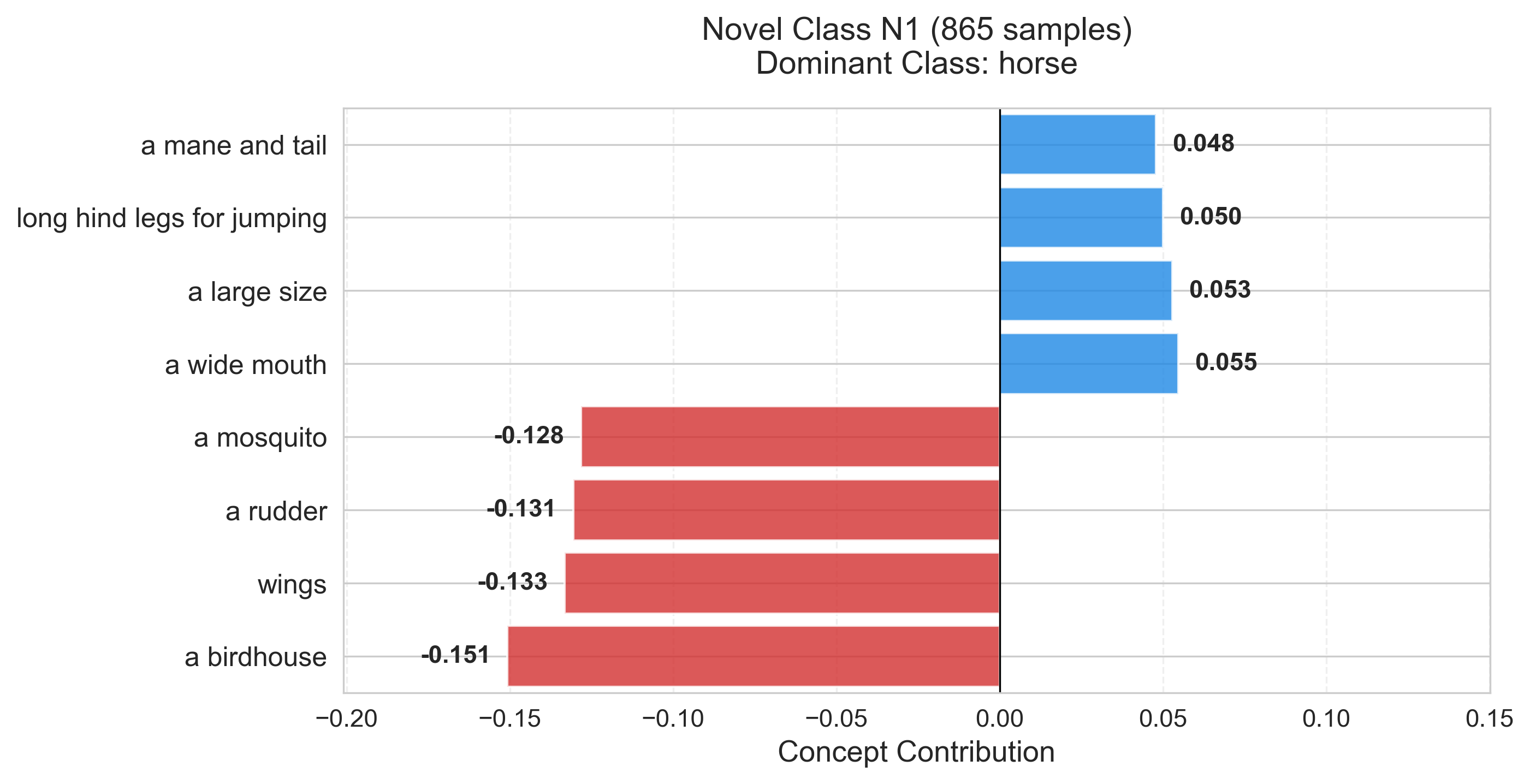}
    \includegraphics[width=0.49\linewidth]{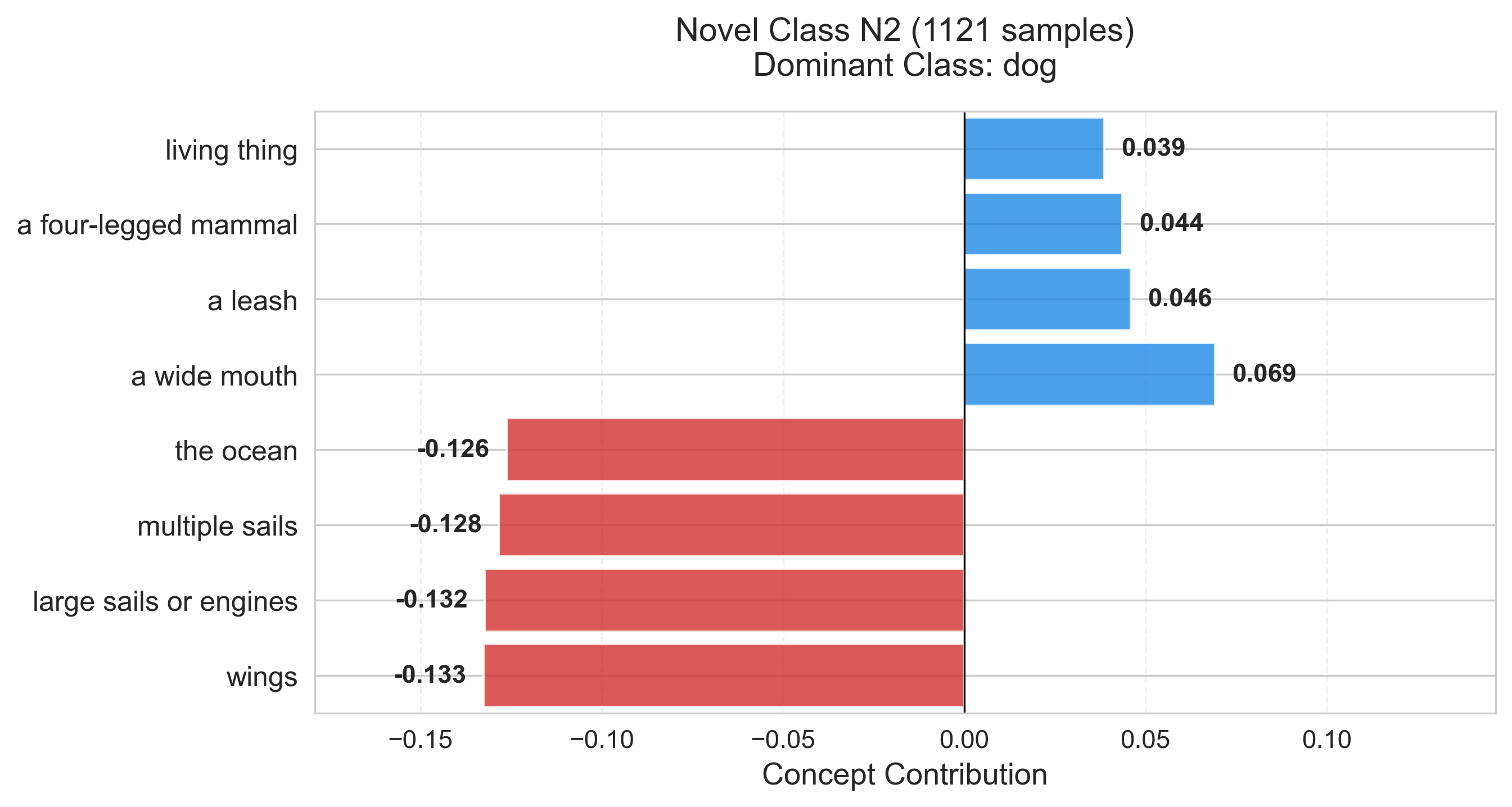}
        \includegraphics[width=0.49\linewidth]{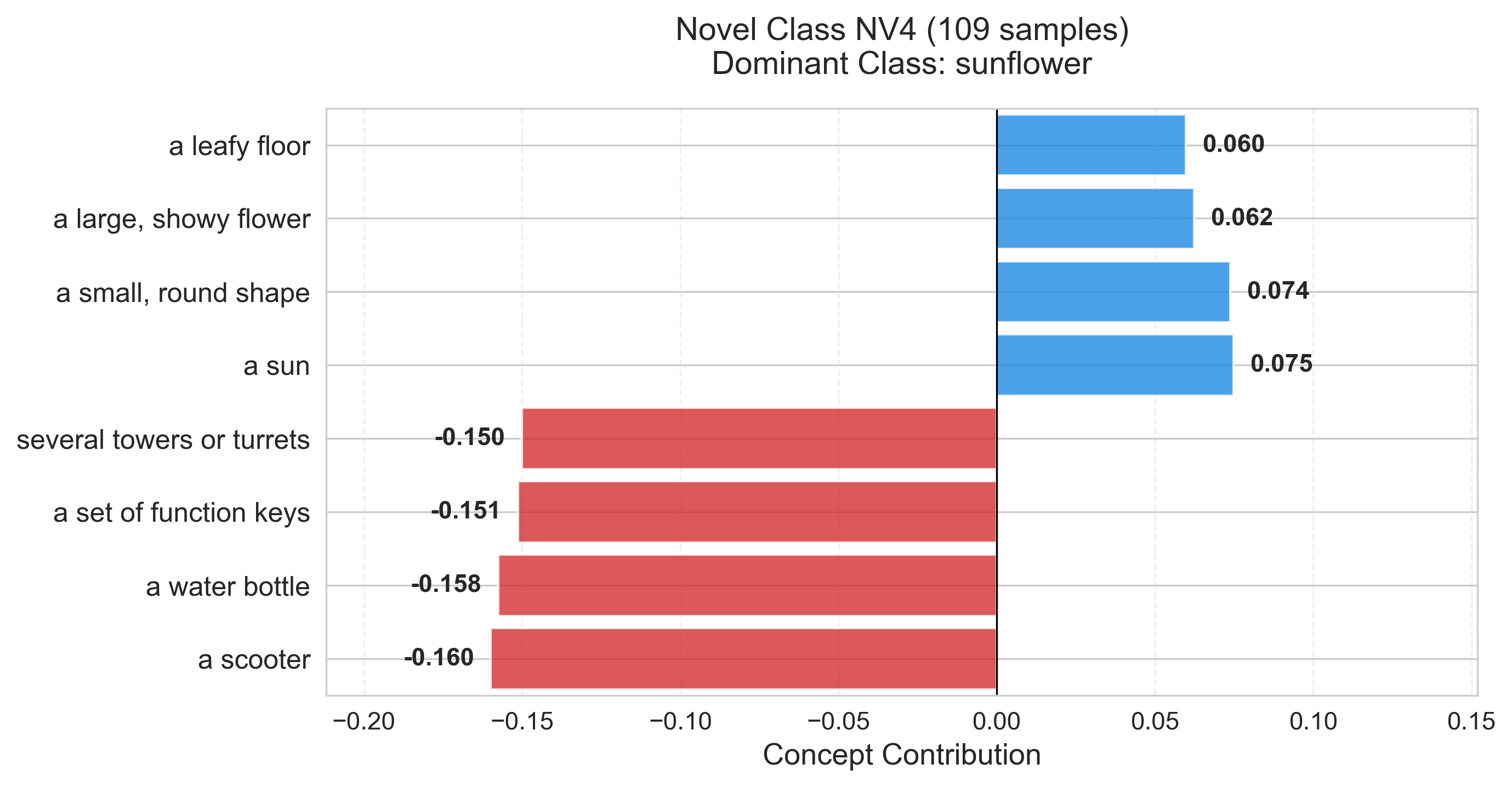}
            \includegraphics[width=0.49\linewidth]{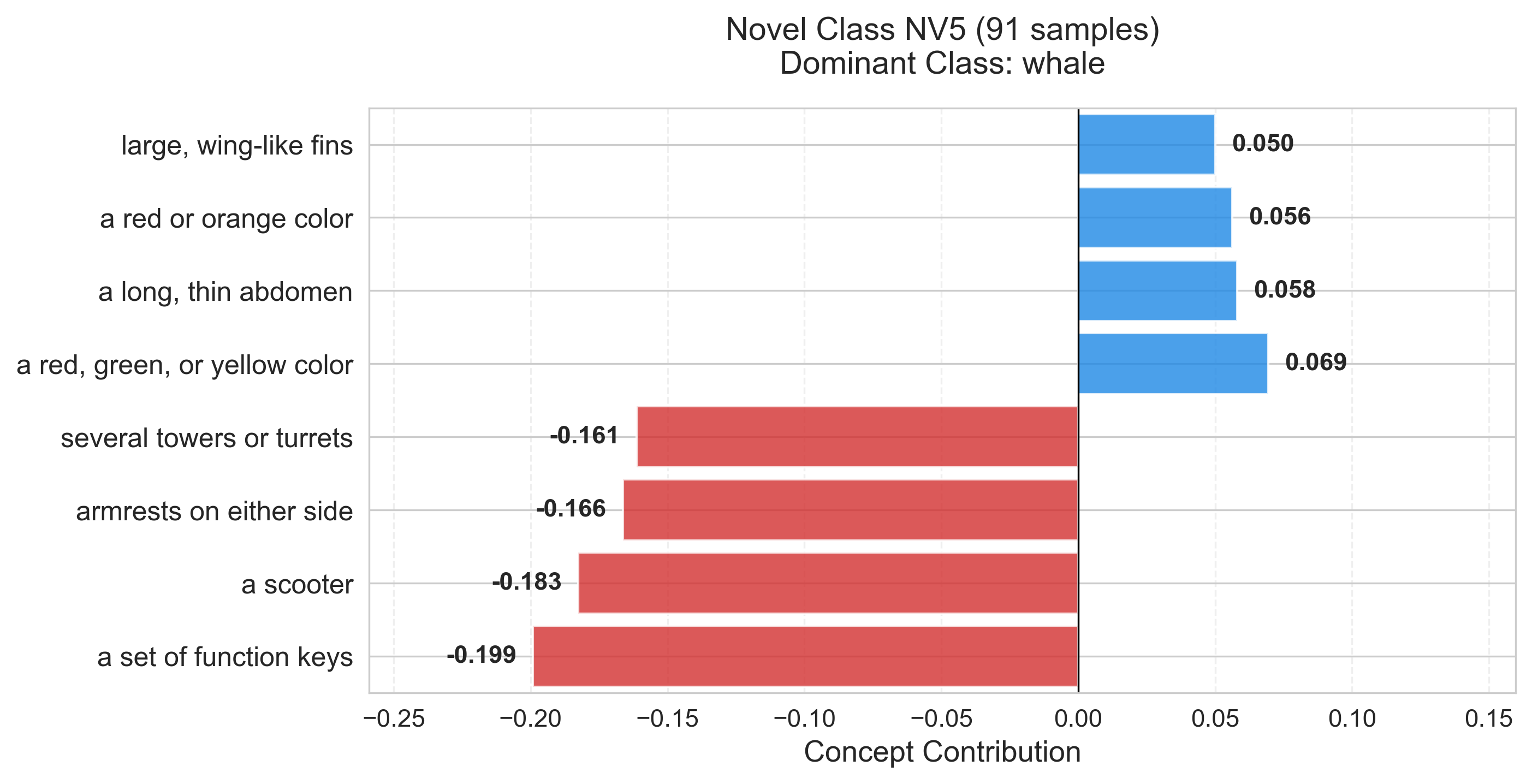}
              \includegraphics[width=0.49\linewidth]{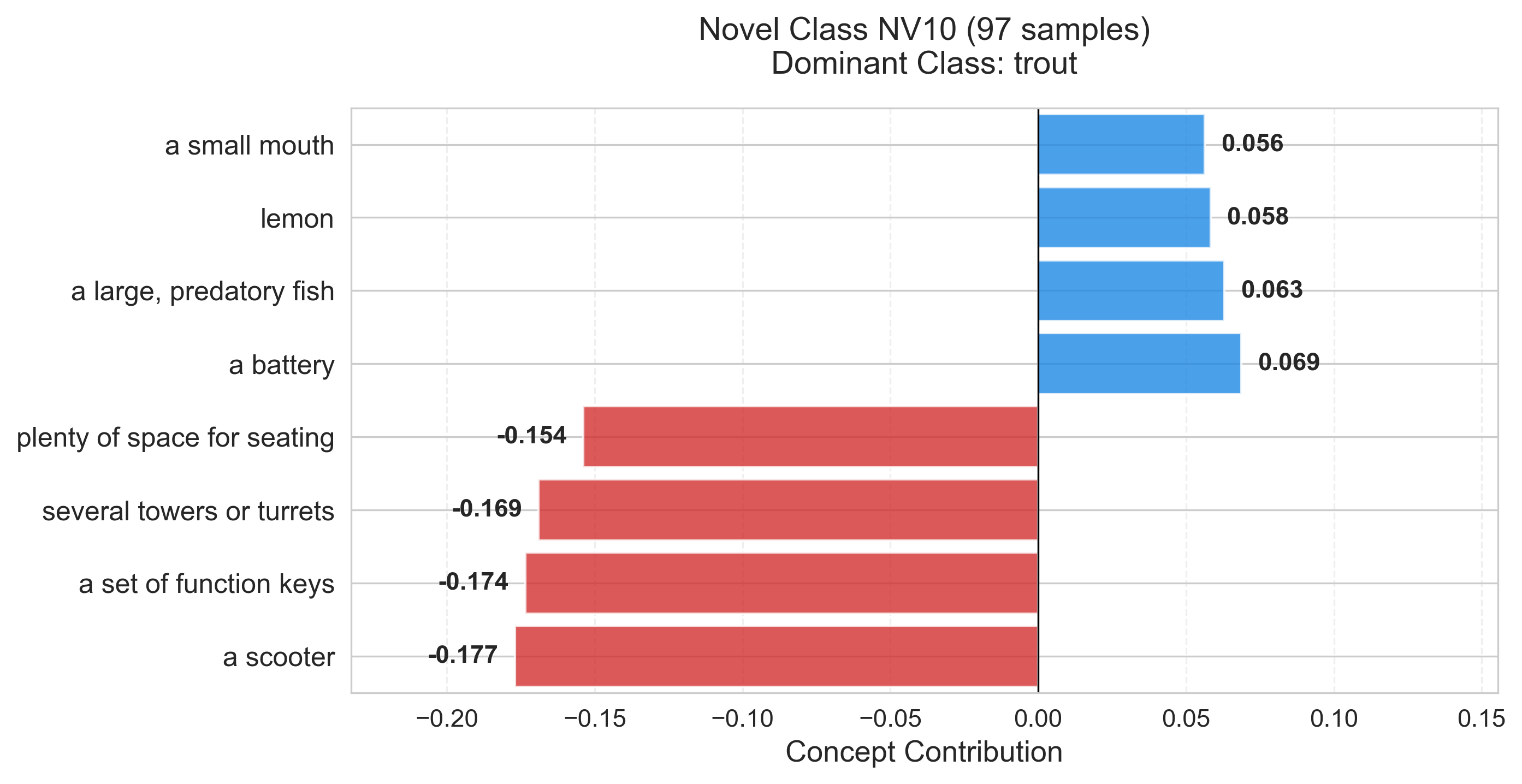}
\caption{Additional cluster-level concept profiles for discovered novel categories}

    \label{fig:app-concept-p}
\end{figure*}

\clearpage
\subsection{Instance level Explanation}

\begin{figure*} [htb]
    \centering

 \includegraphics[width=0.3\linewidth]{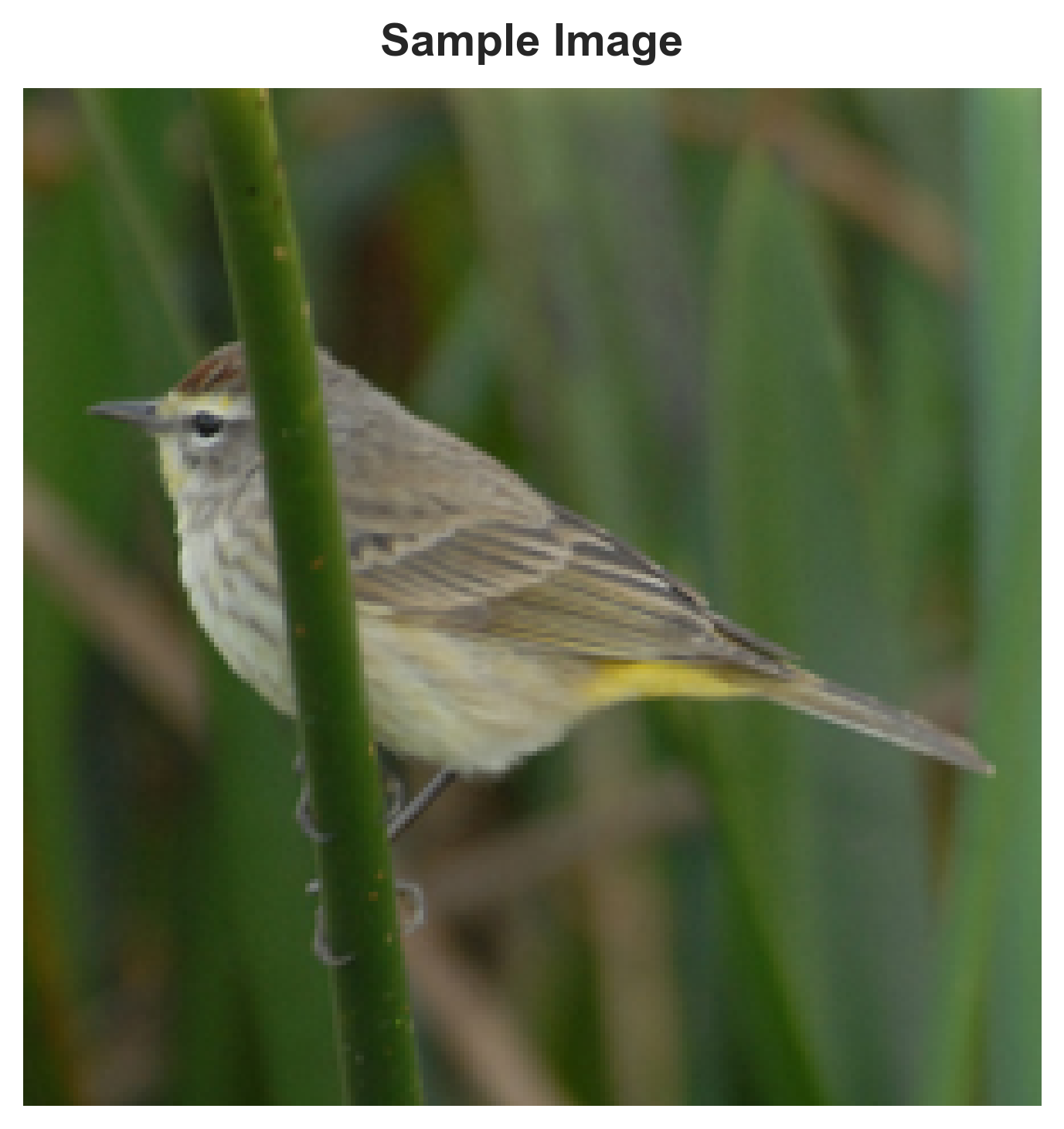}
    \includegraphics[width=0.6\linewidth]{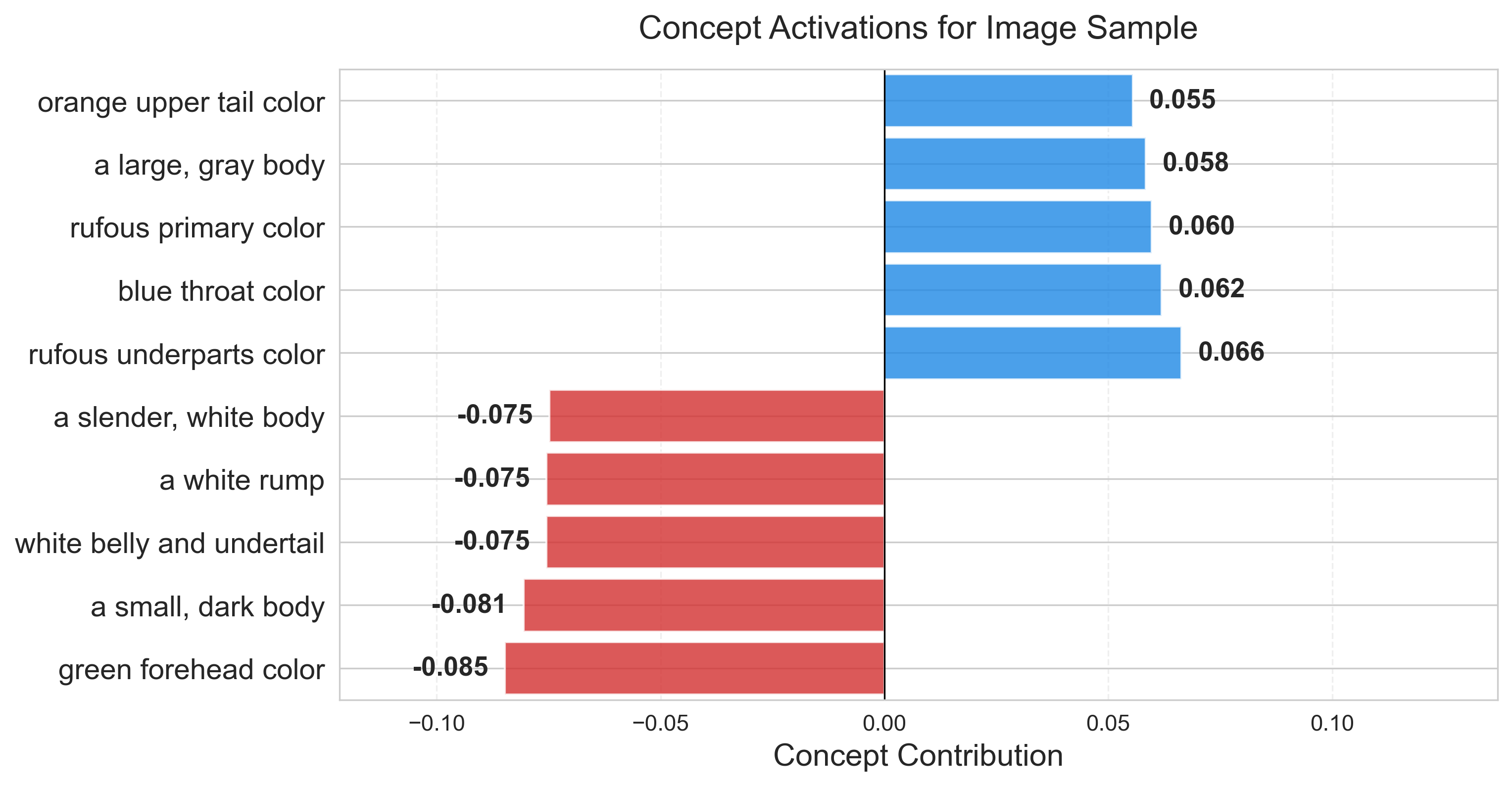}

 \includegraphics[width=0.3\linewidth]{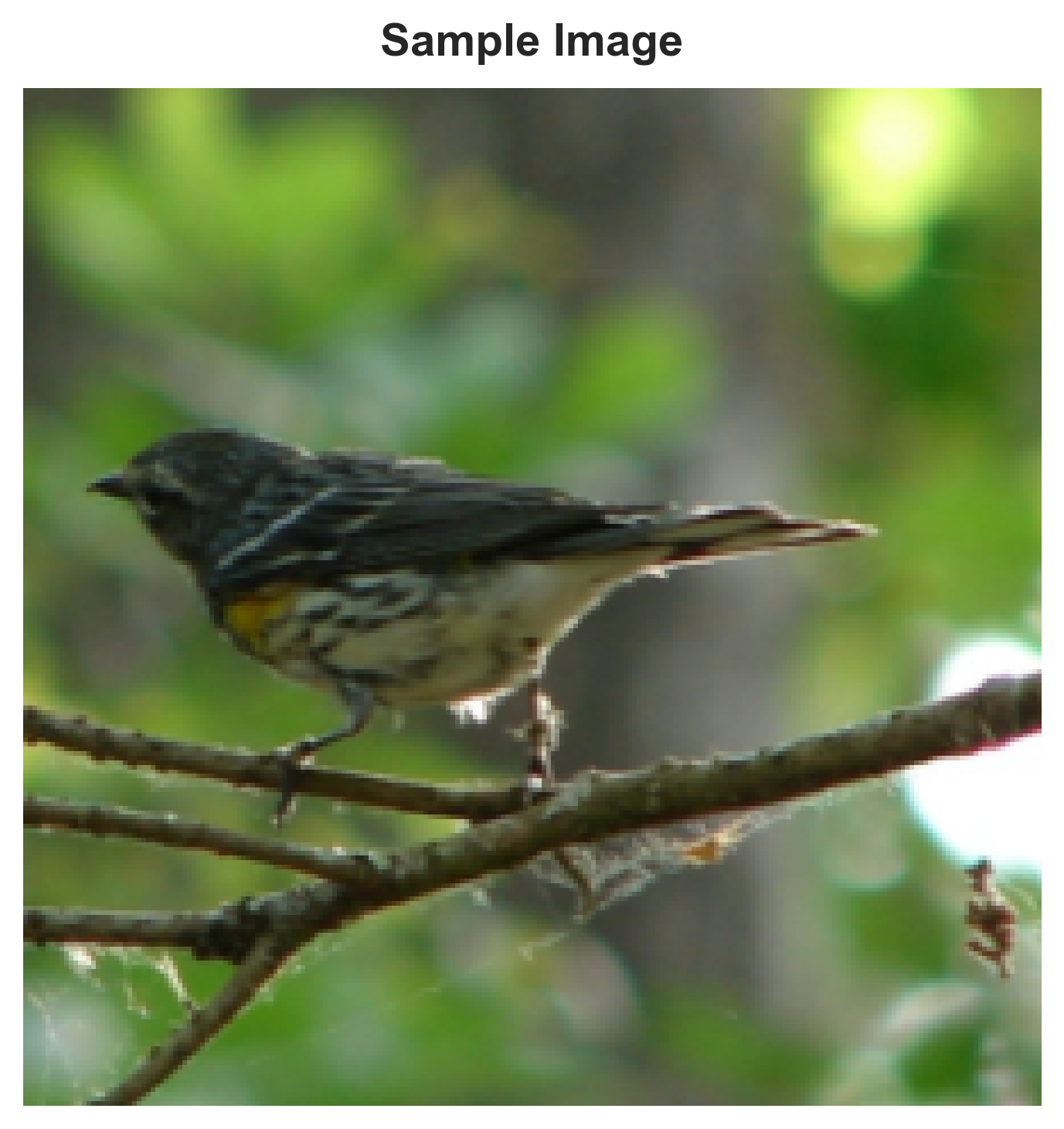}
    \includegraphics[width=0.6\linewidth]{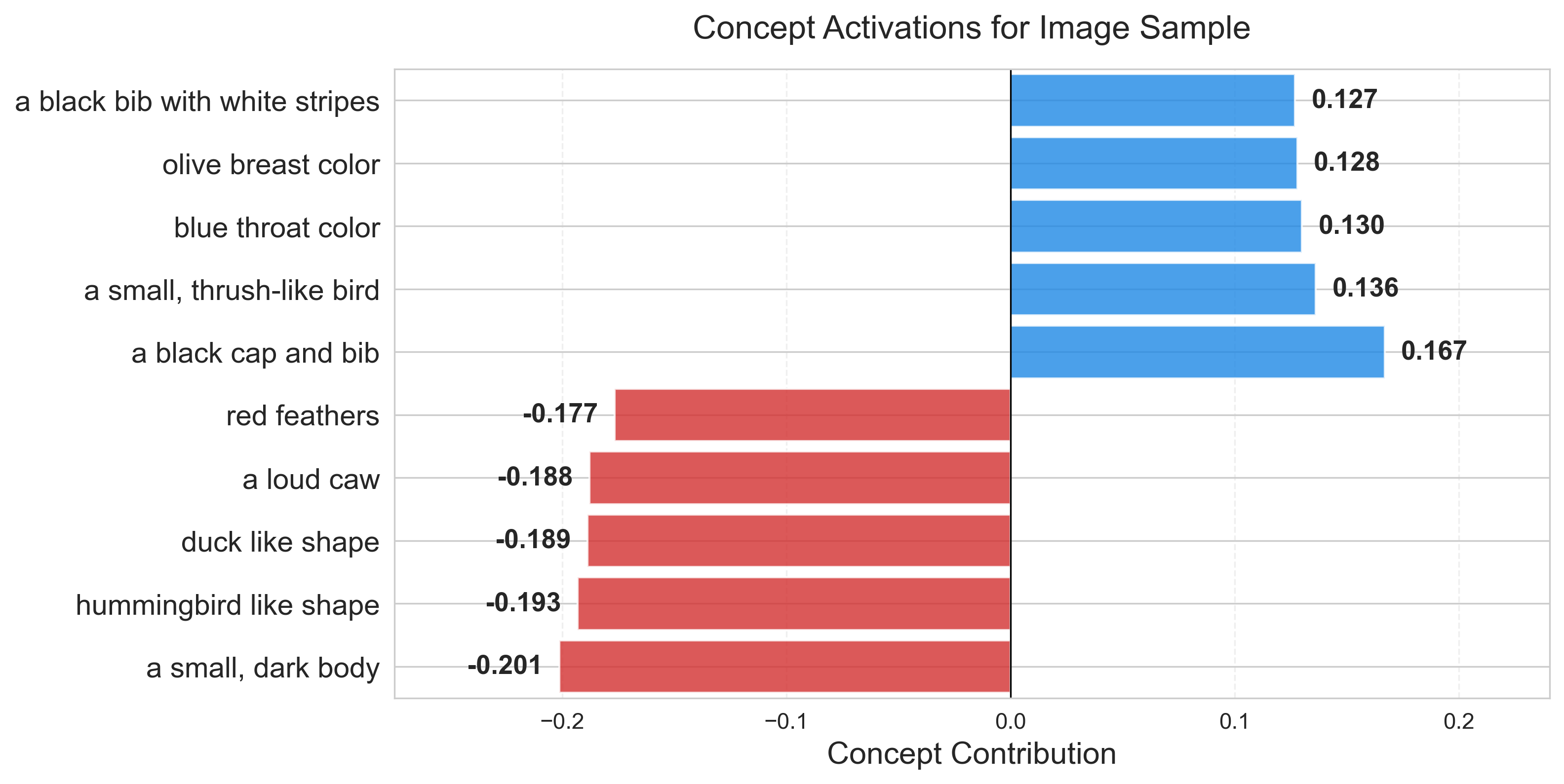}
    
    \includegraphics[width=0.3\linewidth]{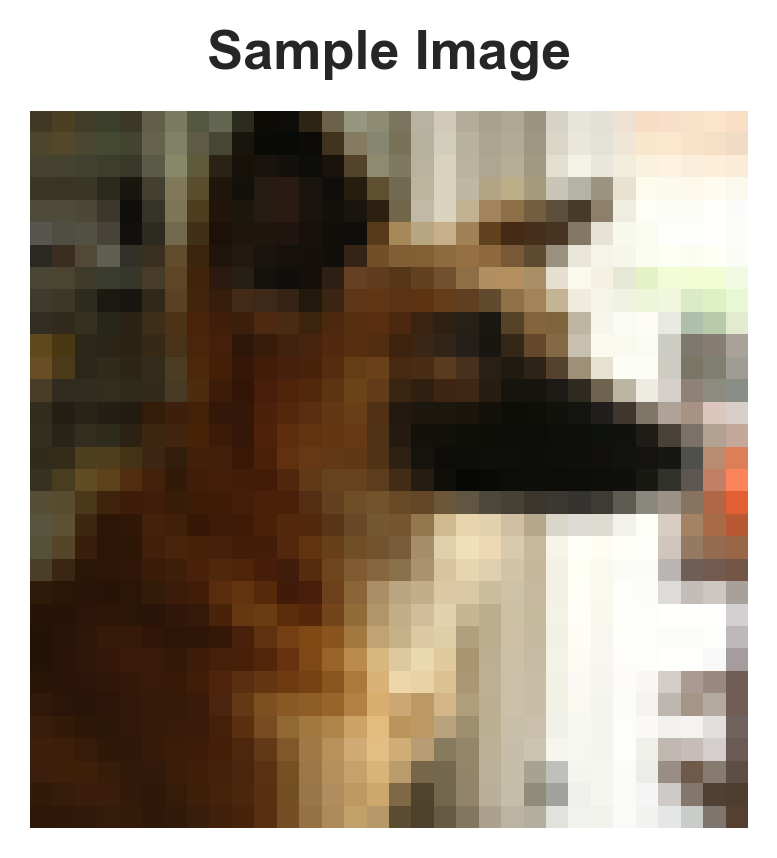}
    \includegraphics[width=0.6\linewidth]{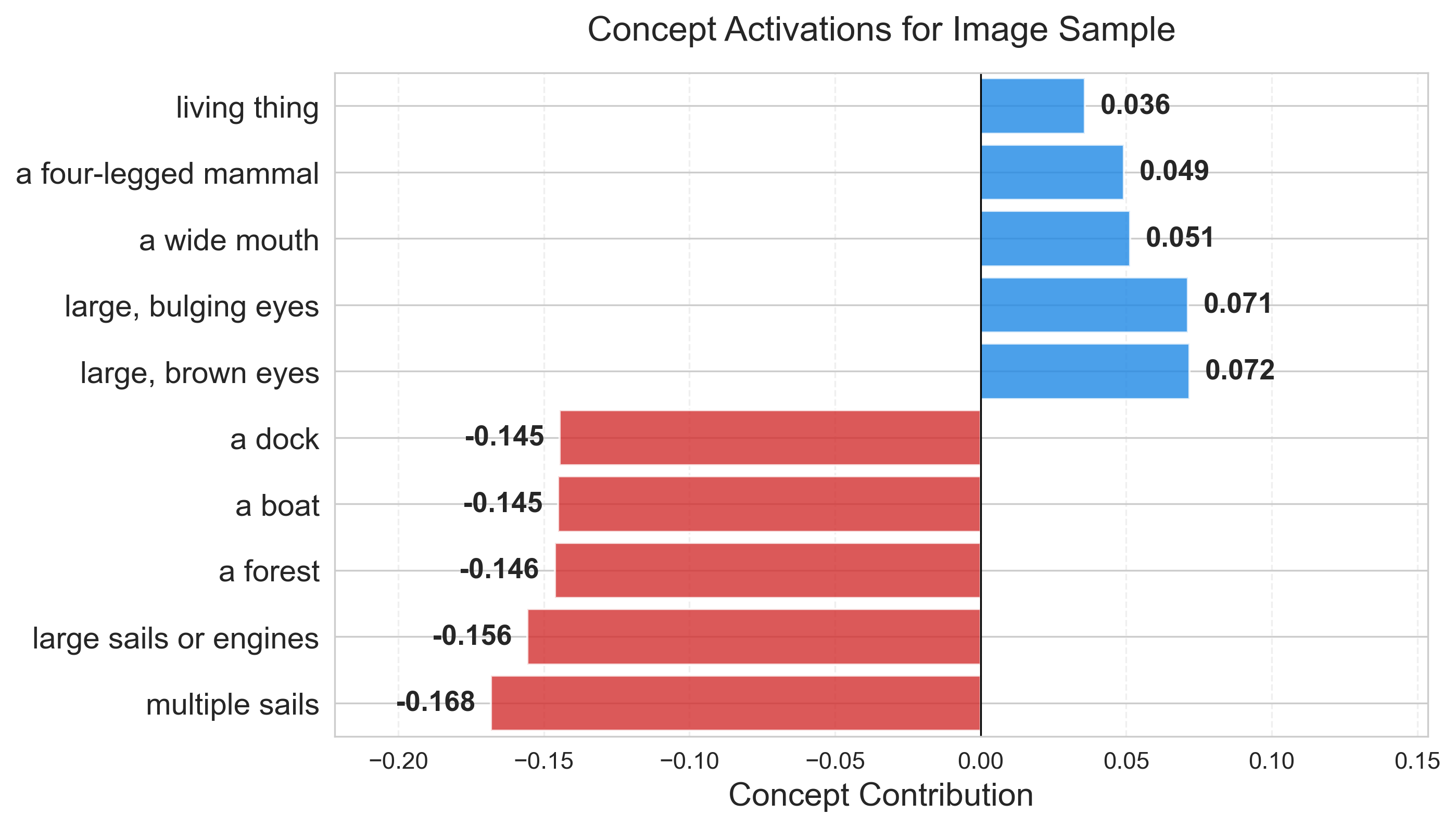}
    
    \includegraphics[width=0.3\linewidth]{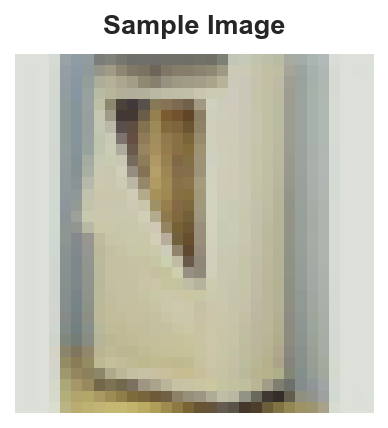}
    \includegraphics[width=0.6\linewidth]{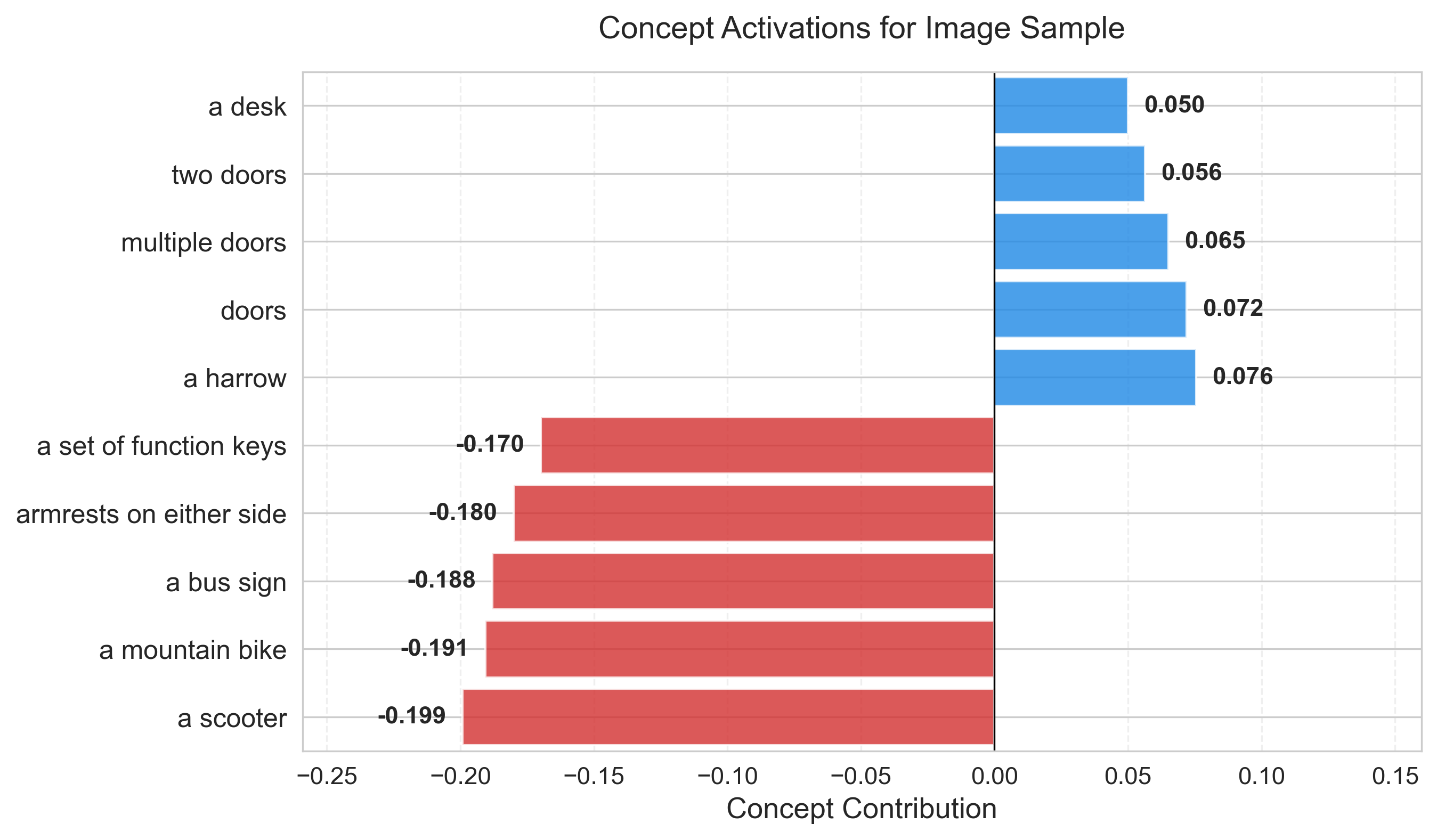}

\caption{Additional instance-level concept explanations for samples from CUB-200, CIFAR-10, and CIFAR-100.}

    \label{fig:app-instance-level}
\end{figure*}

\section{Experiments Computing Resources}
\label{app:exp-resource}
All experiments were performed on NRP Nautilus HPC \citep{10.1145/3708035.3736060}. Training and evaluation used a single GPU for CIFAR10/100 and two GPU for CUB-200 dataset per experiment. Node specifications are provided in Table~\ref{resource}.

\begin{table}[!ht]
\centering
\caption{Hardware configuration for experiments}
\small
\begin{tabular}{l|l}
\toprule
\textbf{Component} & \textbf{Configuration} \\
\midrule
CPUs & Dual 12-core SkyLake 5000 series \\
GPUs & Nvidia RTXA6000 64GB \\
RAM & 32GB \\
SSD & 240GB \\
\bottomrule
\end{tabular}
\label{resource}
\end{table}


\newpage
\section*{NeurIPS Paper Checklist}

\begin{enumerate}

    \item {\bf Claims}
        \item[] Question: Do the main claims made in the abstract and introduction accurately reflect the paper's contributions and scope?
        \item[] Answer: \answerYes{} 
        \item[] Justification: {The abstract and introduction accurately reflect the paper's 
    contributions. The claims cover (i) concept-space NCD framework, (ii) competitive 
    empirical performance on CIFAR-10/100 and CUB-200, and (iii) theoretical 
    hypothesis-space restriction via concept bottleneck, all of which are supported 
    by experiments.}
    \item {\bf Limitations}
        \item[] Question: Does the paper discuss the limitations of the work performed by the authors?
        \item[] Answer: \answerYes{} 
        \item[] Justification: {We have discussed the limitations of our methodology in Appendix \ref{app:discussion}, where we reflect on potential areas for improvement and further exploration}
    
\item {\bf Theory assumptions and proofs}
    \item[] Question: For each theoretical result, does the paper provide the full set of assumptions and a complete (and correct) proof?
  \item[] Answer: \answerYes{} 
    \item[] Justification: {
    Proposition \ref{prop:subset} states assumptions explicitly (linearity of concept map, $K < min\{C, dz\}$) and provides a complete proof in the main paper. The semantic anchoring constraint beyond rank restriction is discussed following the proposition.
    }

    \item {\bf Experimental result reproducibility}
    \item[] Question: Does the paper fully disclose all the information needed to reproduce the main experimental results of the paper to the extent that it affects the main claims and/or conclusions of the paper (regardless of whether the code and data are provided or not)?
    \item[] Answer: \answerYes{} 
    \item[] Justification: {
   All information needed to reproduce the main results is provided 
in the paper, independent of the supplemental code. Full implementation 
details are provided in Appendix~\ref{app:impl-details}, including 
optimizer settings, learning rates, batch sizes, number of epochs, 
Sinkhorn hyperparameters, and concept filtering thresholds. Dataset 
splits are specified in Appendix~\ref{app:dataset-stats}. Algorithms 
for all three stages are provided in Appendix~\ref{app:algo}.
    }

\item {\bf Open access to data and code}
    \item[] Question: Does the paper provide open access to the data and code, with sufficient instructions to faithfully reproduce the main experimental results, as described in supplemental material?
    \item[] Answer: \answerYes{} 
    \item[] Justification: { Anonymized code is provided as supplemental material 
with instructions to reproduce the experimental results. The datasets used (CIFAR-10, CIFAR-100, and CUB-200) are publicly available.}

\item {\bf Experimental setting/details}
    \item[] Question: Does the paper specify all the training and test details (e.g., data splits, hyperparameters, how they were chosen, type of optimizer) necessary to understand the results?
\item[] Answer: \answerYes{} 
    \item[] Justification: { Training and evaluation details are provided in Section \ref{sec:exp-res} 
and Appendix \ref{app:impl-details}, including backbone architecture, optimizer, learning rate 
schedule, temperature, number of clustering heads, Sinkhorn iterations, 
and concept alignment threshold. Dataset splits are described in Appendix \ref{app:dataset-stats}.}

\item {\bf Experiment statistical significance}
    \item[] Question: Does the paper report error bars suitably and correctly defined or other appropriate information about the statistical significance of the experiments?
    \item[] Answer: \answerYes{} 
    \item[] Justification: {All main results in Table\ref{tab:clustering_results} and Table \ref{tab:clustering_metrics} report mean ± standard deviation across 3 runs, capturing variability due to random initialization. 
Error bars are reported consistently across all datasets and evaluation protocols.}

\item {\bf Experiments compute resources}
    \item[] Question: For each experiment, does the paper provide sufficient information on the computer resources (type of compute workers, memory, time of execution) needed to reproduce the experiments?
    \item[] Answer: \answerYes{} 
    \item[] Justification: {In the appendix \ref{app:exp-resource} section of this paper, we have clearly outlined the hardware configuration used to run the experiments for the proposed methods.}
    
\item {\bf Code of ethics}
    \item[] Question: Does the research conducted in the paper conform, in every respect, with the NeurIPS Code of Ethics \url{https://neurips.cc/public/EthicsGuidelines}?
    \item[] Answer: \answerYes{} 
    \item[] Justification: {The research uses publicly available datasets and does not 
involve human subjects, sensitive data collection, or potential for direct harm. 
The work has been conducted in accordance with the NeurIPS Code of Ethics.}

\item {\bf Broader impacts}
    \item[] Question: Does the paper discuss both potential positive societal impacts and negative societal impacts of the work performed?
    \item[] Answer: \answerYes{} 
    \item[] Justification: {Broader impact is discussed in Appendix \ref{app:discussion}, covering positive 
applications in scientific discovery, biodiversity monitoring, and drug discovery. 
The authors do not foresee specific negative societal consequences from this work.}

\item {\bf Safeguards}
    \item[] Question: Does the paper describe safeguards that have been put in place for responsible release of data or models that have a high risk for misuse (e.g., pre-trained language models, image generators, or scraped datasets)?
    \item[] Answer: \answerNA{} 
    \item[] Justification: {The method 
performs category discovery on standard vision benchmarks and poses no 
identified risks requiring safeguards.  As such, it does not involve the release of data or models that carry a high risk of misuse, and therefore, no specific safeguards are required}

\item {\bf Licenses for existing assets}
    \item[] Question: Are the creators or original owners of assets (e.g., code, data, models), used in the paper, properly credited and are the license and terms of use explicitly mentioned and properly respected?
    \item[] Answer: \answerYes{} 
    \item[] Justification: {The datasets used in this paper (CIFAR-10, CIFAR-100, and CUB-200) are publicly available. We have ensured that the terms of use for these datasets are respected, and the corresponding citations are included in the paper.}

\item {\bf New assets}
    \item[] Question: Are new assets introduced in the paper well documented and is the documentation provided alongside the assets?
    \item[] Answer: \answerYes{} 
    \item[] Justification: {Our work does not create new assets such as datasets, However Anonymized code is submitted as supplemental material.}

\item {\bf Crowdsourcing and research with human subjects}
    \item[] Question: For crowdsourcing experiments and research with human subjects, does the paper include the full text of instructions given to participants and screenshots, if applicable, as well as details about compensation (if any)? 
    \item[] Answer: \answerNA{}{} 
    \item[] Justification: {This paper does not involve crowdsourcing or research 
with human subjects. All experiments are conducted on existing 
publicly available image classification benchmarking dataset.}

\item {\bf Institutional review board (IRB) approvals or equivalent for research with human subjects}
    \item[] Question: Does the paper describe potential risks incurred by study participants, whether such risks were disclosed to the subjects, and whether Institutional Review Board (IRB) approvals (or an equivalent approval/review based on the requirements of your country or institution) were obtained?
    \item[] Answer: \answerNA{}{} 
    \item[] Justification: {Our Work does not involve research with human subjects, and therefore, IRB approval or equivalent was not required.}

\item {\bf Declaration of LLM usage}
    \item[] Question: Does the paper describe the usage of LLMs if it is an important, original, or non-standard component of the core methods in this research? Note that if the LLM is used only for writing, editing, or formatting purposes and does \emph{not} impact the core methodology, scientific rigor, or originality of the research, declaration is not required.
    \item[] Answer: \answerNA{} 
    \item[] Justification:{LLMs are not an important, original, or non-standard 
component of the proposed method. Concept set generation follows the 
established pipeline of \citet{oikarinen2023label}, which uses GPT API 
as a standard component of prior work and not a novel contribution 
of this paper. Any additional LLM usage was limited to writing 
assistance (grammar and editing) and did not impact the scientific 
rigor or originality of the research.}

\end{enumerate}

\end{document}